\documentclass{article}

\usepackage{microtype}
\usepackage{graphicx}
\usepackage{subfigure}
\usepackage{booktabs} 

\usepackage{hyperref}

\usepackage[accepted]{icml2023}

\usepackage{amsmath}
\usepackage{amssymb}
\usepackage{mathtools}
\usepackage{amsthm}

\usepackage[capitalize,noabbrev]{cleveref}

\theoremstyle{plain}

\usepackage[textsize=tiny]{todonotes}

\usepackage{makecell}
\newtheorem{theorem}{Theorem}
\newtheorem{cor}{Corollary}

\newtheorem{lemma}{Lemma}
\newtheorem{prop}{Proposition}
\newtheorem{ass}{Assumption}
\newtheorem{definition}[cor]{Definition}
\usepackage{multirow}
\newcommand{\Norm}[1]{\left\|#1\right\|}
\usepackage{paralist}

\def \O {\mathcal{O}}

\def \v {\mathbf{v}}

\def \x {\mathbf{x}}

\def \E {\mathrm{E}}

\def \x {\mathbf{x}}

\def \1 {\mathbf{1}}

\def \E {\mathbb{E}}

\def \u {\mathbf{u}}
\def \v {\mathbf{v}}

\def \x {\mathbf{x}}

\def \z {\mathbf{z}}

\def \P {\mathcal{P}}
\def \J {\mathcal{J}}
\def \LL {\mathcal{L}}

\def \O {\mathcal{O}}

\def \z {\mathbf{z}}

\def \u {\mathbf{u}}

\def \LL {\mathcal{L}}

\def \P {\mathcal{P}}

\def \E {\mathrm{E}}
\def \x {\mathbf{x}}

\def \D {\mathcal{D}}
\def \z {\mathbf{z}}
\def \u {\mathbf{u}}

\def \v {\mathbf{v}}

\def \P {\mathbb{P}}

\def \E {\mathbb{E}}

\icmltitlerunning{Learning Unnormalized Statistical Models via Compositional Optimization}

\begin{document}

\twocolumn[
\icmltitle{Learning Unnormalized Statistical Models via Compositional Optimization}



\icmlsetsymbol{equal}{*}

\begin{icmlauthorlist}
\icmlauthor{Wei Jiang}{nju}
\icmlauthor{Jiayu Qin}{sch}
\icmlauthor{Lingyu Wu}{nju}
\icmlauthor{Changyou Chen}{sch}
\icmlauthor{Tianbao Yang}{comp}
\icmlauthor{Lijun Zhang}{nju}
\end{icmlauthorlist}

\icmlaffiliation{nju}{National Key Laboratory for Novel Software Technology, Nanjing University, Nanjing, China}
\icmlaffiliation{comp}{Department of Computer Science and Engineering, Texas A\&M University, College Station, USA}
\icmlaffiliation{sch}{ Department of Computer Science and Engineering, University at Buffalo, New York, USA}

\icmlcorrespondingauthor{Changyou Chen}{changyou@buffalo.edu}
\icmlcorrespondingauthor{Tianbao Yang}{tianbao-yang@tamu.edu}
\icmlcorrespondingauthor{Lijun Zhang}{zhanglj@lamda.nju.edu.cn}

\icmlkeywords{noise-contrastive estimation, stochastic compositional optimization, unnormalized statistical models}

\vskip 0.3in
]



\printAffiliationsAndNotice{}  

\begin{abstract}
Learning unnormalized statistical models (e.g., energy-based models) is computationally challenging due to the complexity of handling the partition function. To eschew this complexity,  noise-contrastive estimation~(NCE) has been proposed by formulating the objective as the logistic loss of the real data and the artificial noise. However, as found in previous works, NCE may perform poorly in many tasks due to its flat loss landscape and slow convergence. In this paper, we study {\it a direct approach for optimizing the negative
log-likelihood} of unnormalized models from the perspective of compositional optimization. To tackle the partition function,  a noise distribution is introduced such that the log partition function can be written as a compositional function whose inner function can be estimated with stochastic samples. Hence, the objective can be optimized by stochastic compositional optimization algorithms. Despite being a simple method, we demonstrate that it is more favorable than NCE by (1) establishing a fast convergence rate  and quantifying its dependence on the noise distribution through the variance of stochastic estimators; 
(2) developing better results for one-dimensional Gaussian mean estimation by showing our objective has a much favorable loss landscape and hence our method enjoys faster convergence;
(3) demonstrating better performance on multiple applications, including density estimation, out-of-distribution detection, and real image generation. 
\end{abstract}

\section{Introduction}
We investigate the problem of learning unnormalized statistical models. Suppose we observe a set of training samples $\D_n=\{\x_1, \ldots, \x_n\}$ from an unknown probability density function~(pdf) $p_{\text{data}}(\x)$ and  estimate this data pdf by
\begin{align}\label{eqn:1}
p(\x;\theta) = \frac{p_{0}(\x;\theta)}{\int p_{0}(\x;\theta) d\x},
\end{align}
where $p_{0}(\x;\theta)$ is defined as an unnormalized model, and $\theta$ denotes the parameter that will be learnt to best fit the data.
The term $\int p_{0}(\x;\theta) d\x$ in equation~(\ref{eqn:1}) is called partition function, which is used to ensure the final model is normalized, i.e., $\int p(\x;\theta) d\x = 1$. By introducing the partition function, we can use more flexible structures to represent  $p_{0}(\x;\theta)$. Specifically, if we set $p_{0}(\x;\theta) =  e^{f_0(\x;\theta)}$, the above formulation~(\ref{eqn:1}) reduces to the well-known energy-based model~(EBM)~\cite{LeCun2006ATO}:
\begin{align*}
p(\x;\theta) = \frac{e^{f_0(\x;\theta)}}{\int e^{f_0(\x;\theta)}d\x},
\end{align*}
which enjoys wide applications in machine learning~\cite{NEURIPS2019_378a063b,pmlr-v119-yu20g,Grathwohl2020Your}.

To find the best parameter $\theta$ for the given model $p_0(\x;\theta)$, we can directly maximize the log-likelihood on the observed training data, i.e., optimizing the objective $\LL(\theta)$, where
\begin{align}\label{obj:mle}
\LL(\theta) = - \frac{1}{n}\sum_{i=1}^n \left[ \log p_{0}(\x_i;\theta)\right] + \log \int p_{0}(\x;\theta) d\x.  
\end{align}
However, the challenge is that the log partition function $\log \int p_{0}(\x;\theta) d\x$ and its gradient are difficult to calculate exactly. To remedy this issue, prior works~\cite{Hinton2002, Tieleman2008TrainingRB,Nijkamp2019OnTA} resort to Markov chain Monte Carlo~(MCMC) sampling~\cite{MCMC}. Considering the gradient of loss $\LL(\theta)$, we have:
\begin{align*}
\nabla \LL(\theta) = -\frac{1}{n}\sum_{i=1}^n \left[\frac{\nabla p_{0}(\x_i;\theta)}{p_{0}(\x_i;\theta)} \right] + \E_{\x \sim p_{\theta}} \left[\frac{\nabla p_{0}(\x;\theta)}{p_{0}(\x;\theta)}\right],
\end{align*}
where $p_{\theta}$ denotes the pdf $p(\x,\theta)$. To compute the second term, previous literature applies a number of MCMC steps to sample from $p_{\theta}$, and then calculates the estimated gradient. However, MCMC sampling is slow and unstable during training~\cite{pmlr-v119-grathwohl20a, quiqi2020, geng2021bounds}, partly because approximate samples are obtained with only a finite number of steps.  As pointed out by \citet{NEURIPS2019_2bc8ae25} and \citet{pmlr-v119-grathwohl20a}, estimating with finite MCMC steps will produce biased estimations, leading to optimizing a different objective other than the original MLE loss.

 To eschew the complexity of estimating the gradient of log partition function, noise-contrastive estimation~(NCE)~\cite{JMLR:v13:gutmann12a} has been proposed, which transforms the original problem into a classification task. Specifically, NCE introduces a noise distribution $q(\x)$, and then optimizes the extended model parameter $\tau=(\theta, \alpha)$ by minimizing the following loss:
\begin{align*}
\J(\tau) = -\E_{\x \sim p_\text{data}}\log h(\x, \tau) -\E_{\x \sim q} \log (1 -h(\x, \tau)),
\end{align*}
where $h(\x, \tau) = 1/(1+e^{-( \log p_0(\x, \theta)  -\log 
 q(\x) - \alpha)})$, and the parameter $\alpha$ is used to estimate the log partition function, i.e., $\alpha = \log \int p_{0}(\x;\theta) d\x$. This new objective can be interpreted as the logistic loss to distinguish between the noise and the real data. In practice, the first term is estimated by $n$ observed training examples and the second term is evaluated by $m=\nu n$ noise samples where $\nu\geq 1$.  

However, as pointed out in many works~\cite{NIPS2014_5ca3e9b1, quiqi2020}, if the noise distribution $q$ is very different from the data distribution $p_\text{data}$, this classification problem would be too easy to solve and the learned model may fail to capture adequate information about the real data distribution. In particular, \citet{liu2022analyzing} demonstrate that when $q$ and $p_\text{data}$ are not close enough, the loss $\J(\tau)$ is extremely flat near the optimum, leading to the slow convergence rate of NCE method.

In this paper, we investigate an alternative approach for  directly optimizing the MLE objective by converting it to a stochastic compositional optimization (SCO) problem. To deal with the intractable partition function, we introduce a noise distribution $q(\x)$, and then convert the log partition function $\log \int p_{0}(\x;\theta) d\x$ into $\log \E_{\x \sim q} \left[ \frac{p_{0}(\x;\theta)}{q(\x)}\right]$. 
Then, the objective function~(\ref{obj:mle}) becomes:
\begin{align*}
 \LL(\theta) = & - \frac{1}{n}\sum_{i=1}^n \left[ \log p_{0}(\x;\theta) \right]+  \log \E_{\x \sim q} \left[ \frac{p_{0}(\x;\theta)}{q(\x)}  \right],
\end{align*}
which can be written as a two-level SCO problem. Since SCO has been studied extensively, state-of-the-art algorithms can be employed to solve the above problem. However, besides its simplicity,  a major question remains: \textit{What are the advantages of this approach compared with NCE and MCMC-based methods for learning unnormalized models?} Our main contributions are to demonstrate the following advantages by theoretical analysis and empirical studies: 
\begin{compactenum}
\item We prove that our single-loop algorithm converges asymptotically to an optimal solution under the Polyak-\L ojasiewicz~(PL) condition, and  establish a fast convergence rate in the order of $O(1/\epsilon)$ for finding an $\epsilon$-optimal solution. In contrast, NCE and MCMC-based approaches do not provide such guarantees.

\vspace*{0.05in}\item We prove through one-dimensional Gaussian mean estimation that the MLE objective function has a better loss landscape than the NCE objective, and establish a faster convergence rate of our algorithm than a state-of-the-art NCE-based approach~\cite{liu2022analyzing} for this task. 

\vspace*{0.05in}\item We illustrate the better performance of our method on different tasks, namely, density estimation, out-of-distribution detection,  and real image generation. For the last task, we show that the choice of the noise distribution has a significant impact on the quality of generated data in line with our theoretical analysis.  
\end{compactenum}

\section{Related Work}
This section reviews related work on noise contrastive estimation, and other methods for learning unnormalized models, as well as stochastic compositional optimization.
\subsection{Noise-contrastive Estimation}
Noise-contrastive estimation~(NCE) is first proposed by ~\citet{pmlr-v9-gutmann10a, JMLR:v13:gutmann12a}, and gains its popularity in machine learning applications quickly~\cite{NIPS2013_db2b4182,hjelm2018learning,pmlr-v119-henaff20a,Multiview,Kong2020A}. The basic idea is to introduce a noise distribution, and then distinguish it from the data distribution by a logistic loss. As pointed out in previous works~\cite{NIPS2014_5ca3e9b1, quiqi2020, liu2022analyzing}, the choice of the noise distribution is crucial to the success of NCE, and many methods have been proposed to tune the noise automatically. For example, \citet{bose-etal-2018-adversarial} utilize a learned adversarial distribution as the noise, and develop a method named Adversarial Contrastive Estimation. At the same time, Conditional NCE is introduced by~\citet{pmlr-v80-ceylan18a}, which  generates the noise samples with the help of the observed data. Afterward, \citet{quiqi2020} propose Flow Contrastive Estimation, using a flow model to learn the noise distribution by joint training. However, these methods usually introduce extra computation, and the adversarial training of the noise is slow and complex. Instead of designing more complicated noise, \citet{liu2022analyzing} recently propose a new objective named eNCE, which replaces the log loss in NCE as the exponential loss, and enjoys  better results for exponential families. However, eNCE still suffers from the ill-behaved loss landscape, which is extremely flat near the optimum.

\subsection{Other Methods for Learning Unnormalized Models}
Except NCE and its variants discussed above, there are also many other approaches to solve unnormalized models. To bypass the partition function, score matching~\cite{JMLR:v6:hyvarinen05a} optimizes the squared distance between the gradient of the log density of the model and that of the observed data, in which the partition function would not appear. 
However, the output dimension of score matching is the same as the input, and training such a model is expensive and unstable, especially for high-dimensional data~\cite{NEURIPS2019_3001ef25, NEURIPS2020_de6b1cf3}. Existing works usually require denoising~\cite{Vincent2011} or slicing~\cite{song2019} techniques to train score matching methods.   

Besides NCE and score matching,  optimizing MLE objective with Markov chain Monte Carlo (MCMC) sampling is also widely used. One well-known method is Contrastive Divergence~(CD) introduced by \citet{Hinton2002}, which employs MCMC sampling for fixed steps. To sample more efficiently, \citet{Tieleman2008TrainingRB} initializes the MCMC in each training step with the previous sampling results and names this method as persistent CD. Other variants of CD have also been proposed later, such as modified CD~\cite{8579052}, adversarial CD~\cite{HanNFHZW19}, etc. These techniques are also very popular in learning energy-based models~(EBM), which is a special case of unnormalized models. Specifically, \citet{add4} propose to learn EBM by using ConvNet as the energy function; then, \citet{add3} and \citet{add5} use a generator as a fast sampler to help EBM training. Note that the objective of the EBM becomes a modified contrastive divergence, and \citet{add7} and \citet{add6} are proposed as different variants of \citet{add5}, whose objectives are also a modified contrastive divergence. However, the main drawback of these methods is that MCMC sampling is usually time-consuming and unstable during training~\cite{pmlr-v119-grathwohl20a, quiqi2020, geng2021bounds}.

\subsection{Stochastic Compositional Optimization}
Stochastic Compositional Optimization~(SCO) has been investigated extensively in the literature. The objective of a two-level SCO is given by $\mathbb E_{\xi_1}[f_{\xi_1}(\mathbb E_{\xi_2}[g_{\xi_2}(\theta)])]$, where $\xi_1$ and $\xi_2$ are random variables. To solve this problem, \citet{wang2017stochastic} develop stochastic compositional gradient descent (SCGD), which achieves a complexity of $\mathcal{O}\left(\epsilon^{-7}\right)$,  $\mathcal{O}\left(\epsilon^{-3.5}\right)$, and $\mathcal{O}\left(\mu^{-14/4} \epsilon^{-5/4}\right)$ for non-convex, convex, and $\mu$-strongly convex functions, respectively. These complexities are further improved to $\mathcal{O}\left(\epsilon^{-4.5}\right)$, $\mathcal{O}\left( \epsilon^{-2}\right)$ and $\mathcal{O}\left(\epsilon^{-1}\right)$ in a subsequent work \citep{DBLP:journals/jmlr/WangLF17}.

To obtain a better rate, \citet{Ghadimi2020AST} propose a method named averaged stochastic approximation (NASA), which uses momentum-update 
to estimate the inner function and the gradient, and attains a complexity of  $\mathcal{O}\left(\epsilon^{-4}\right)$ for non-convex objectives. 
Recently, with the popularity of variance reduction methods such as SARAH~\citep{arxiv.1703.00102}, SPIDER~\cite{Fang2018SPIDERNN} and STORM~\citep{cutkosky2019momentum}, 
this complexity is further improved to $\mathcal{O}\left(\epsilon^{-3}\right)$, by estimating the inner function and the gradient with variance-reduction techniques~\citep{Zhang2019ASC,chen2021solving,qi2021online}.

\section{The Proposed Method}
In this section, we propose to learn unnormalized models directly by Maximum likelihood Estimation via Compositional Optimization~(MECO). First, it is easy to show that, with a noise distribution $q(\x)$, the MLE objective function~(\ref{obj:mle}) is equivalent to:
\begin{align*}
 \LL(\theta) = & - \frac{1}{n}\sum_{i=1}^n \left[ \log p_{0}(\x_i;\theta) \right]+  \log \E_{\x \sim q} \left[ \frac{p_{0}(\x;\theta)}{q(\x)}  \right].
\end{align*}
However, optimizing this objective is nontrivial, because of the nested structure of the second term. Specifically, we can not acquire its unbiased estimation by sampling from $q(\x)$, since the expectation cannot be moved out of the log function, i.e., $\E_{\x \sim q} \left[ \log   \frac{p_{0}(\x;\theta)}{q(\x)}  \right] \neq \log \E_{\x \sim q} \left[ \frac{p_{0}(\x;\theta)}{q(\x)}  \right]$. Due to similar reasons, we also can not obtain an unbiased estimation of its gradient, which is the main obstacle to deriving an algorithm with convergence guarantees.

To solve this difficulty, we can treat the $ \log \E_{\x \sim q} \left[ \frac{p_{0}(\x;\theta)}{q(\x)}  \right]$ as a  compositional function, where $\log$ is the outer function and $\E_{\x \sim q} \left[ \frac{p_{0}(\x;\theta)}{q(\x)}  \right]$ is the inner function, which can be unbiased-estimated by sampling form $q(\x)$.
Inspired by NASA~\cite{Ghadimi2020AST} for solving stochastic compositional optimization~(SCO), we use variance reduced estimators to approximate the inner function and the gradient, so that the estimation errors can be reduced over time. Specifically, considering the gradient of the objective function w.r.t. parameter $\theta$, we have:
\begin{align*}
\nabla \LL(\theta) = & - \frac{1}{n}\sum_{i=1}^n \left[ \frac{\nabla p_{0}(\x;\theta)}{p_{0}(\x;\theta)}\right]+  \frac{\E_{\x \sim q} \left[  \frac{\nabla p_{0}(\x;\theta)}{q(\x)}  \right]}{\E_{\x \sim q} \left[ \frac{p_{0}(\x;\theta)}{q(\x)}  \right]}.
\end{align*}
To estimate this gradient, in each training step $t$, we first draw sample $\z_t$ from the training set $\{\x_1,\cdots,\x_n\}$, and sample $\widetilde{\z}_t$ from the noise distribution $q(\x)$. Then, we estimate $\E_{\x \sim q} \left[ \frac{p_{0}(\x;\theta)}{q(\x)}  \right]$ by a function value estimator $\u_t$ in the style of momentum update:
\begin{align}\label{u}
    \u_{t} = (1-\gamma_t) \u_{t-1} + \gamma_t  \frac{p_{0}(\widetilde{\z}_t;\theta_t)}{q(\widetilde{\z}_t)}.
\end{align}
When evaluating the gradient, we would use $\u_t$ to approximate the term $\E_{\x \sim q} \left[ \frac{p_{0}(\x;\theta)}{q(\x)}  \right]$ in the denominator. Similarly, we employ a gradient estimator $\v_t$ to track the overall gradient $\nabla \LL(\theta)$ by another momentum update:
\begin{align}
\begin{split} \label{v}
    \v_{t} = &(1-\beta_t) \v_{t-1} \\
    & + \beta_t \left(-\frac{\nabla p_{0}(\z_t;\theta_t)}{p_{0}(\z_t;\theta_t)}+  \frac{1}{\u_t} \frac{\nabla p_{0}(\widetilde{\z}_t;\theta_t)}{q(\widetilde{\z}_t)} \right).
\end{split}
\end{align}
Although the estimators $\u_t$ and $\v_t$ are still biased, it can be proved that the estimation error is reduced gradually. After obtaining the gradient estimator $\v_t$, we use it to update the parameter $\theta_t$ in the style of SGD. The whole algorithm is presented in Algorithm~\ref{alg:multi}. Note that in the first iteration, we can simply set the estimator $\u_1 =\frac{p_{0}(\widetilde{\z}_1;\theta_1)}{q(\widetilde{\z}_1)}$ and $\v_1 = -\frac{\nabla p_{0}(\z_1;\theta_1)}{p_{0}(\z_1;\theta_1)}+ \frac{1}{\u_1} \frac{\nabla p_{0}(\widetilde{\z}_1;\theta_1)}{q(\widetilde{\z}_1)}$. 
\begin{algorithm}[tb]
	\caption{MECO}
	\label{alg:multi}
	\begin{algorithmic}
	\STATE {\bfseries Input:} time step $T$, initial points $(\theta_1, \u_1, \v_1)$ \\
         \quad\quad\quad  sequence $\{ \eta_t,\gamma_t, \beta_t \}$
		\FOR{time step $t = 1$ {\bfseries to} $T$}
		\STATE Sampling $\z_t$ from $\{\x_1,\cdots,\x_n\}$ and $\widetilde{\z}_t$ from $q(\x)$
		\STATE Update estimator $\u_{t}$ according to equation~(\ref{u})
		\STATE Update estimator $\v_{t}$ according to equation~(\ref{v})
		\STATE Update the weight: $\theta_{t+1} = \theta_t - \eta_t \v_{t}$
		\ENDFOR
	\STATE Choose $\tau$ uniformly at random from $\{1, \ldots, T\}$
	\STATE Return $\theta_{\tau}$
	\end{algorithmic}
\end{algorithm}

\textbf{Difference from NASA:} The optimization method we used can be viewed as a modified version of NASA for SCO problems~\cite{Ghadimi2020AST}. Compared with the original NASA applied to constrained SCO, our method does not need to project the variable $\theta$ onto the feasible set and is thus simpler. Another {\bf big difference} from NASA lies in the improved rate we will derive for an objective satisfying the PL condition, which is missing in their original work. 

\textbf{Difference from energy-based model~(EBM) training:} Derived from SCO, our algorithm has a key difference from the standard EBM training. By setting $p_0(\x;\theta) = e^{f_0(\x;\theta)}$, the unnormalized model converts to the EBM, and our gradient estimator $\v_t$ is written as:
\begin{align*}
(1-\beta_t)\v_{t-1} + \beta_t \left(-\nabla f(\z_t;\theta_t)+  \frac{e^{f (\widetilde{\z}_t;\theta_t)}}{  q(\widetilde{\z}_t)\u_t} \nabla f (\widetilde{\z}_t;\theta_t) \right).
\end{align*}
In contrast, EBM usually optimizes the objective function $\mathcal{L}^\prime (\theta) = - \mathbb{E}_{\z \sim p_\text{data}}\left[f(\z; \theta)\right] +  \mathbb{E}_{\widetilde{\z} \sim q}\left[f(\widetilde{\z}; \theta)\right]$, and its stochastic gradient is written as
$\nabla \mathcal{L}^\prime (\theta_t) = - \nabla f(\z_t; \theta_t) + \nabla f(\widetilde{\z}_t; \theta_t)$.
In this sense, our method can be viewed as introducing an adaptive weight $\frac{e^{f (\widetilde{\z}_t;\theta_t)}}{ q(\widetilde{\z}_t) \u_t} $  to the second term and applying SGD with momentum to update the model.

\section{Theoretical Analysis}
We first present the advantage of the MLE formulation, and then analyze the convergence rate of the proposed method, as well as its relationship with the noise distribution. Due to space limitations, all the proofs are deferred to the appendix.
\subsection{Advantages of the MLE Formulation}
Since we use MLE to optimize unnormalized models, our objective inherits several nice properties. Here, we analyze the behavior of the estimator $\widehat{\theta} = \arg \min_{\theta} \LL(\theta)$ for large sample sizes and assume that there exists an optimal solution $\theta^{*}$ such that $p(\x, {\theta^*})=p_{\text{data}}(\x)$. To compare different estimators, we introduce the definition of asymptotic relative efficiency~(ARE)~\cite{vaart_1998} below.

\begin{definition}
    For any two estimators $R$ and $S$ with
    \begin{align*}
        \sqrt{n}\left(R-\theta^{*}\right) \rightsquigarrow N\left(0, r^2\right), \
        \sqrt{n}\left(S-\theta^{*}\right) \rightsquigarrow N\left(0, s^2\right),
    \end{align*}
the ARE of $S$ to $R$ is defined by $ARE(S, R) =r^2/s^2$.
\end{definition}
\textbf{Remark:} From the definition, we know that $ARE(S, R) < 1$ indicates estimator $R$ is more efficient than estimator $S$.
\begin{theorem}\label{thm:1} According to the property of MLE~\cite{Wasserman}, the estimator $\widehat{\theta}$ enjoys the following guarantees:
\begin{compactenum}
    \item (Consistent) The estimator $\widehat{\theta}$ converges in probability to $\theta^{*}$, i.e., $\widehat{{\theta}} \stackrel{P}{\rightarrow} {\theta}^{*}$.
    \item (Asymptotically Normal) $\sqrt{n}\left(\widehat{\theta}-\theta^{*}\right) \rightsquigarrow N(0,{\widehat{\mathrm{se}}}^2)$, where $\widehat{\mathrm{se}}$ can be computed analytically. 
    \item (Asymptotically Optimal) 
    Denote that $\widetilde{\theta}$ is the output of any other estimator, then $A R E\left(\widetilde{\theta}, \widehat{\theta}\right) \leq 1$.
\end{compactenum}
\end{theorem}
\textbf{Remark:}  The last property implies that the MLE objective has the smallest variance, indicating it is better than optimizing other objectives, e.g., the  NCE objective. 

\subsection{The Convergence of the Proposed Method}
Then, we analyze the convergence of  Algorithm~\ref{alg:multi}. First, we define the sample complexity to measure the convergence rate, which is widely used in stochastic optimization. 
\begin{definition}
The sample complexity is the number of samples needed to find a point satisfying $\E \left[\Norm{\nabla \LL(\theta)}\right] \leq \epsilon$ ($\epsilon$-stationary), or $\E \left[ \LL(\theta)-\inf_{\theta} \LL(\theta)\right] \leq \epsilon$ ($\epsilon$-optimal).
\end{definition}
For notation simplicity, we denote $g(\theta) = \E_{\x \sim q} \left[ \frac{p_{0}(\x;\theta)}{q(\x)} \right]$, $h(\theta) = - \frac{1}{n}\sum_{i=1}^n \left[ \log p_{0}(\x_i;\theta) \right]$, $f(\cdot)=\log(\cdot)$,  estimator $g(\theta; \widetilde{\z}) = \frac{p_{0}(\widetilde{\z};\theta)}{q(\widetilde{\z})}$ and $h(\theta; \z) = -\log p_{0}(\z;\theta)$, where $\widetilde{\z}$ and $\z$ are samples drawn from $q$ and $\D_n=\{\x_1,\cdots,\x_n\}$. 

Next, we make the following assumptions, which are commonly used in SCO problems~\citep{wang2016accelerating,Zhang2019ASC,Zhang2021MultiLevelCS,pmlr-v162-wang22ak,ICML:2022:Jiang}.
\begin{ass}\label{asm:1}  We assume that (i) function $f(\cdot)$, $g(\cdot)$ are  Lipchitz continuous and smooth with respect to their inputs, $h(\cdot)$ is a smooth function; (ii) $\LL$ is lower bounded by $\LL^{*}$.
\end{ass}
\begin{ass}\label{asm:2} There exist $\sigma_g, \zeta_g, \zeta_h$  such that 
\begin{align*}
\mathbb{E}_{\widetilde{\z}\sim q(\widetilde{\z})}\left[\left\|g(\theta; \widetilde{\z})-g(\mathbf{\theta})\right\|^{2}\right] \leq \sigma_g^{2},  \\
\mathbb{E}_{\widetilde{\z}\sim q(\widetilde{\z})}\left[\left\|\nabla g(\theta; \widetilde{\z}) - \nabla g(\mathbf{\theta})\right\|^{2}\right] \leq \zeta_g^{2}, \\
\mathbb{E}_{\z\sim\D_n}\left[\left\|\nabla {h}(\theta; \z)-\nabla h(\theta)\right\|^{2}\right] \leq \zeta_h^{2}.    
\end{align*}
\end{ass}
\textbf{Remark:} In Assumption~\ref{asm:1}, $f(\cdot)$ is Lipchitz and smooth in terms of its input when $p_0(\widetilde{\z};\theta)/q(\widetilde{\z}) \geq c$, where $c$ is a positive constant. Assumption~\ref{asm:2} assumes that the variances of estimating $g(\theta)$, $\nabla g(\theta)$ and $h(\theta)$ by sampling from the corresponding distributions are bounded.

We first derive an asymptotic convergence result, showing that Algorithm~\ref{alg:multi} can converge to a stationary point of $\LL(\theta)$. 
\begin{theorem} \label{thm:4}
Assume that the sequence of stepsizes satisfies
 $\sum_{t=1}^{\infty} \eta_t = + \infty, \sum_{t=1}^{\infty} \eta_t^2 <  \infty$.
Then with probability 1,  accumulation point $(\theta^{\prime}$,$\u^{\prime}$,$\v^{\prime})$ of the sequence $\{\theta_t$,$\u_t$,$\v_t\}$ generated by Algorithm~\ref{alg:multi} satisfies: 
\begin{align*}
    \nabla \LL(\theta^{\prime}) = 0, \quad u^{\prime} = g(\theta^{\prime}), \quad \ \v^{\prime} = \nabla \LL(\theta^{\prime}).
\end{align*}
\end{theorem}
Next, we present the non-asymptotic convergence result. 
\begin{theorem} \label{thm:2}
For non-convex function $\LL(\cdot)$, by setting $\eta_t = \mathcal{O}(\epsilon^2)$, $\beta_t = \gamma_t = \mathcal{O}(\epsilon^2)$, Algorithm~\ref{alg:multi} finds an $\epsilon$-stationary point in $T = \mathcal{O}\left(\max\left\{\frac{1}{\epsilon^2},\frac{(\sigma_g^2+\zeta_g^2+\zeta_h^2)}{\epsilon^{4}}\right\}\right)$  iterations.
\end{theorem}
\textbf{Remark:} The analysis of above two theorems closely follows~\citet{Ghadimi2020AST}. 

{\bf The main contribution of our analysis} is to show that by properly setting the hyper-parameters, Algorithm~\ref{alg:multi} can enjoy a faster convergence rate when the objective satisfies the Polyak-Łojasiewicz~(PL) condition~\citep{Karimi2016LinearCO}. We give the definition of the PL condition below.
\begin{definition}
$\LL(\theta)$ satisfies the $\mu$-PL condition if there exists $\mu > 0$ such that $2 \mu\left(\LL(\theta)-\LL^{*}\right) \leq\|\nabla \LL(\theta)\|^{2}$.
\end{definition}
{\bf Remark:} {PL condition has been shown to be satisfied for deep learning under over-parameterized networks by many prior works~\cite{pmlr-v97-allen-zhu19a,du2018gradient}. 
Under the PL condition, a stationary point $\theta'$ becomes a global optimal solution of objective $\LL$. Thus, our algorithm asymptotically converges to an optimal solution $\widehat\theta=\arg\min\LL(\theta)$ according to Theorem~\ref{thm:4}. Next, we show the improved complexity under the PL condition as stated below.} 

\begin{theorem} \label{thm:3}
When the objective satisfies $\mu$-PL condition, by setting  $\gamma_{t+1} =\beta_{t+1} = \mathcal{O}\left(\max\{1,\mu\}\eta_t\right)$ and $1-\mu\eta_t =  \eta_t^2 / \eta_{t-1}^2$,  Algorithm~\ref{alg:multi} can find an $\epsilon$-optimal solution in $T= \mathcal{O}\left( \max\left\{\frac{1}{\mu \sqrt{\epsilon}}, \frac{(\sigma_g^2 +\zeta_g^2 + \zeta_h^2)}{\mu \epsilon}, \frac{(\sigma_g^2 + \zeta_g^2 + \zeta_h^2)}{\mu^2 \epsilon} \right\}\right)$ iterations.
\end{theorem}

\textbf{Remark:} We emphasize that the above convergence for a single-loop algorithm is novel in SCO. Existing  methods for SCO usually have to employ two-loop stagewise methods to obtain similar results under the PL condition~\cite{Zhang2019ASC,qi2021stochastic,ICML:2022:Jiang}. 

\begin{figure*}[ht]
\vskip 0.1in
\begin{center}
\centerline{\subfigure[The NCE objective]{\includegraphics[width=0.32\textwidth]{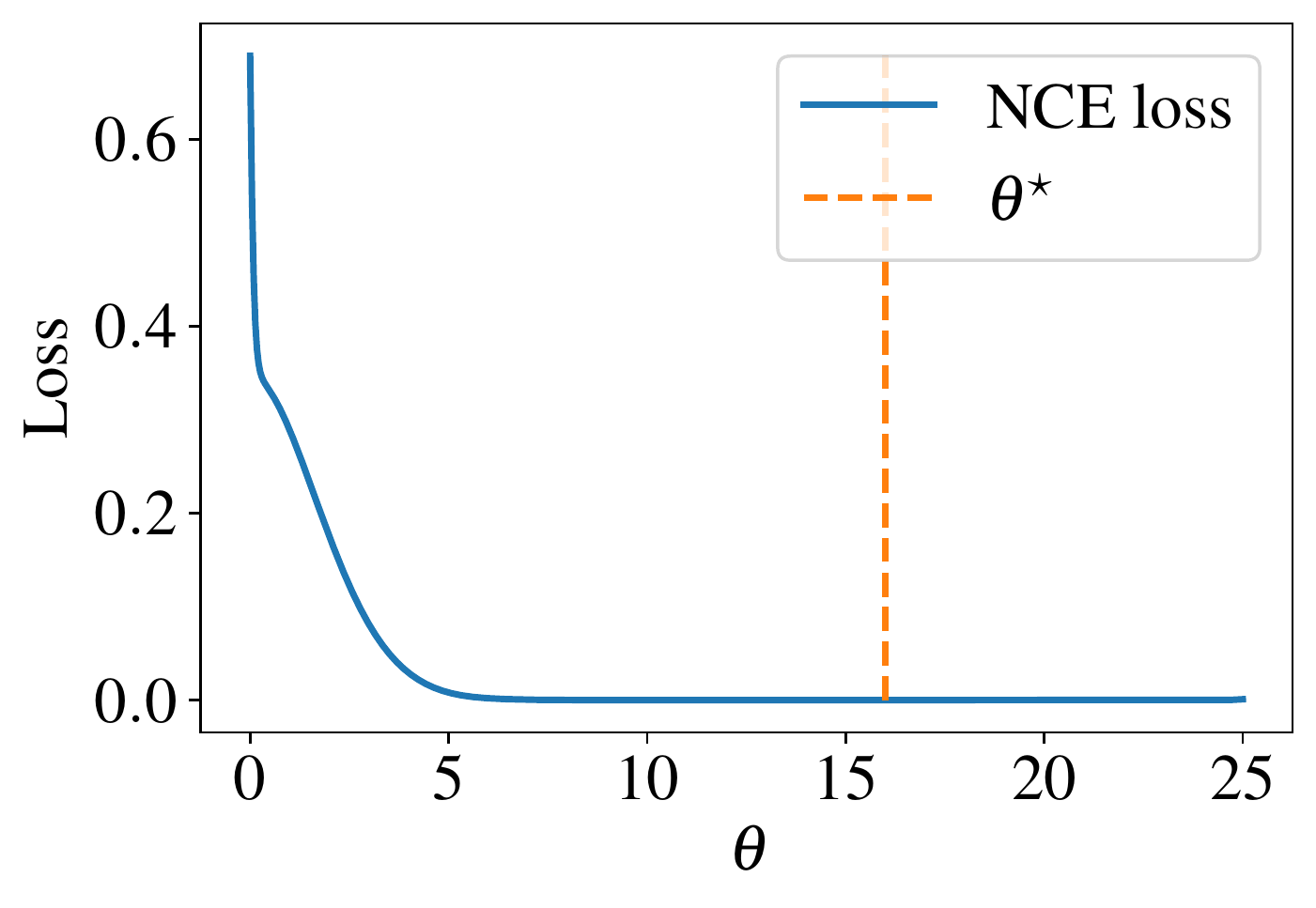}}
\subfigure[The eNCE objective]{\includegraphics[width=0.32\textwidth]{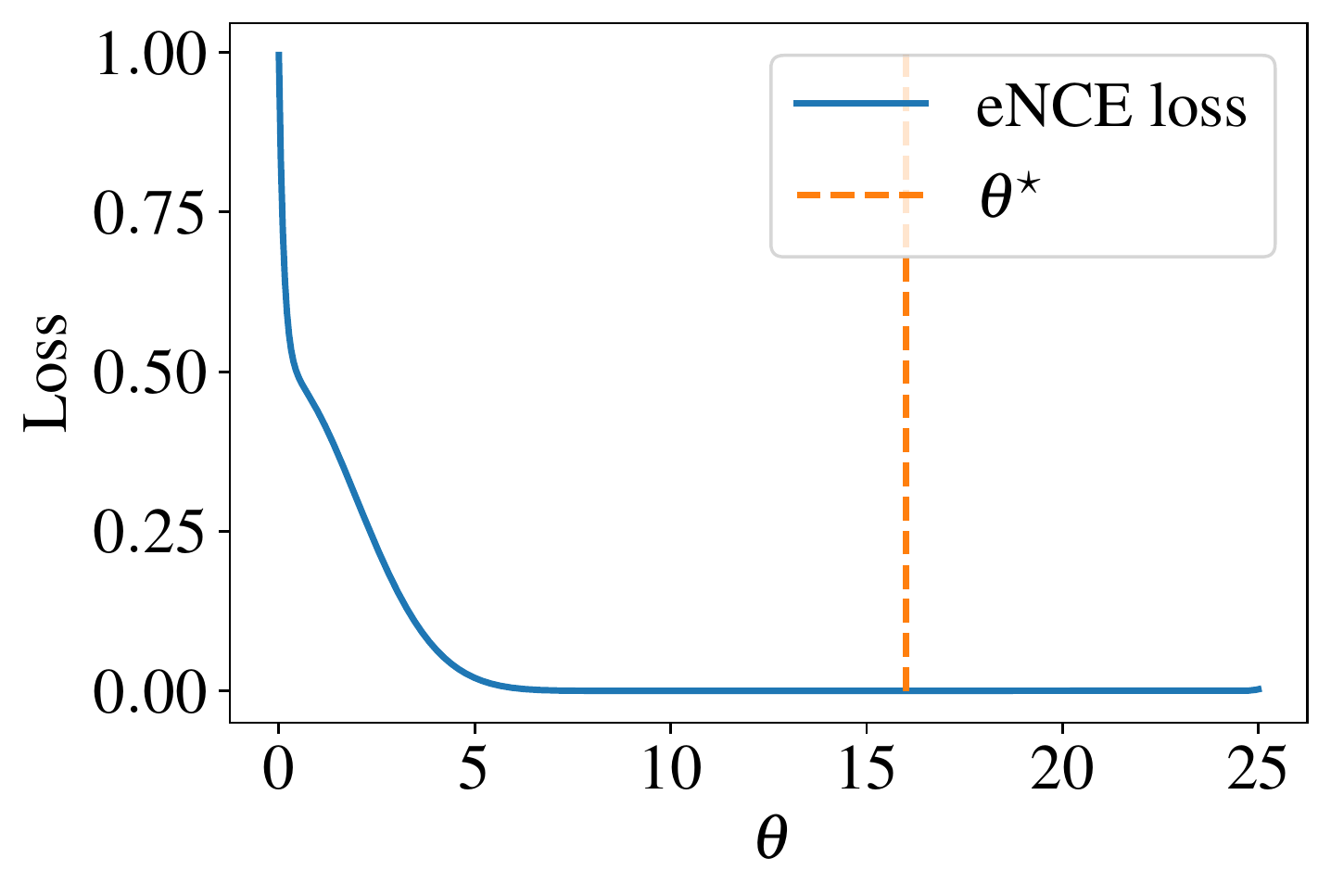}}
\subfigure[The MLE objective (ours)]{\includegraphics[width=0.32\textwidth]{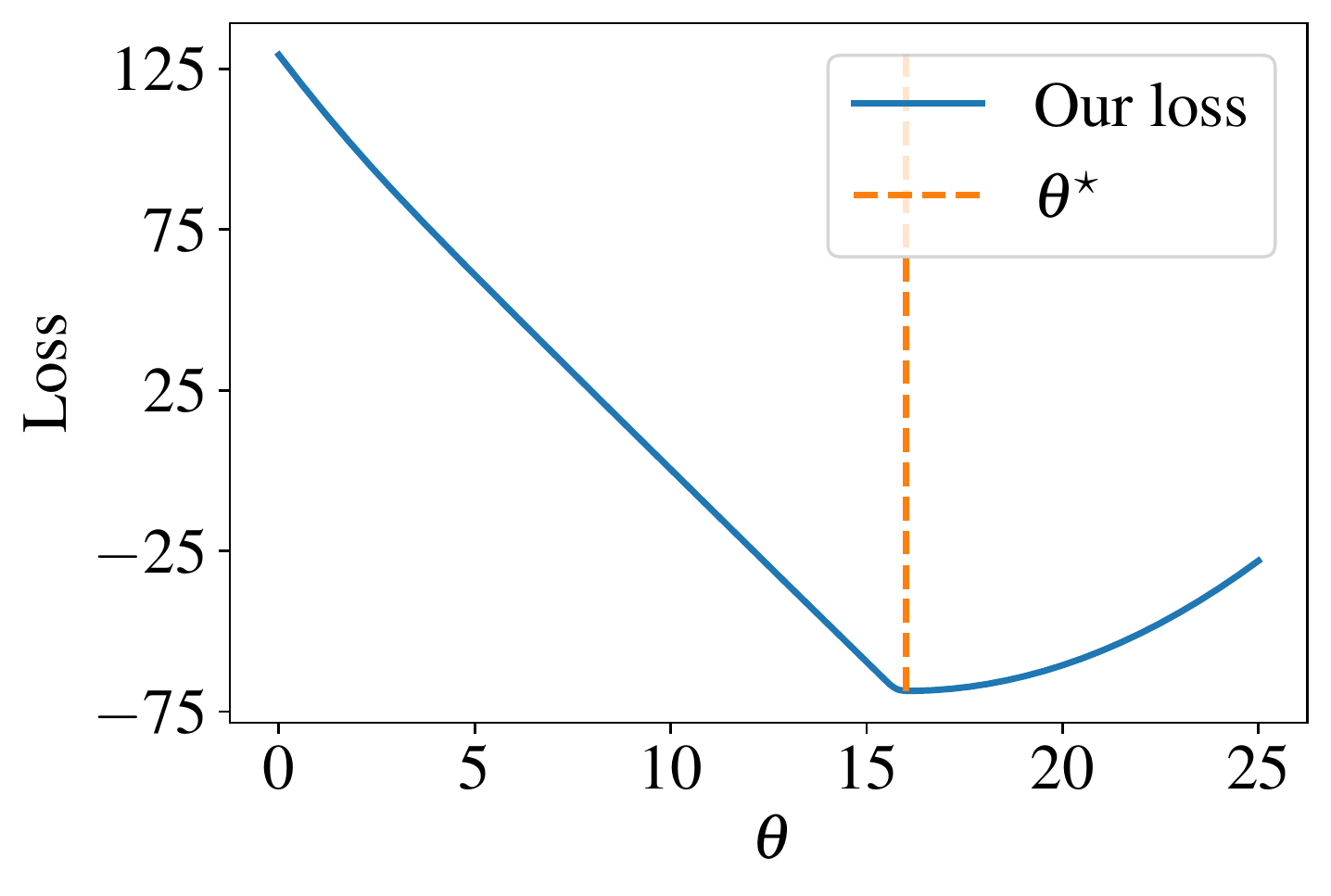}}}
\vskip -0.1in
\caption{The loss landscape of three objectives. Note that the optimal parameter $\theta^{*} = 16$.}
\label{fig:loss}
\end{center}
\vskip -0.3in
\end{figure*}

\subsection{Choosing the Noise Distribution} 
Similar to NCE, our approach also relies on a noise  distribution $q(\widetilde{\z})$. The difference is that the noise distribution only affects the convergence rate of our method, not the optimal solution. Theorem~\ref{thm:4} shows that our algorithm will eventually converge to a stationary point or a global optimal solution (under PL condition) of the MLE objective, as long as $q(\widetilde{\z})$ is positive whenever $p_0(\widetilde{\z};\theta)$ is positive, no matter what the noise distribution is. On the other hand, the noise distribution would affect the objective function of NCE, and it is pointed out that if the noise distribution is too different from data distribution, the classification problem becomes too easy, which  prevents the model from learning much about the data structure~\cite{JMLR:v13:gutmann12a}. For our method, the impact of $q(\widetilde{\z})$  on the convergence rate is through the variance of $g(\theta; \widetilde{\z})$ and $\nabla {g}(\theta; \widetilde{\z})$.    From our theoretical results (Theorems~\ref{thm:2} and~\ref{thm:3}), we can see that if the variance is zero, then the noise distribution has no impact on the convergence rate. 
We show the condition to satisfy the zero variance below.
\begin{lemma}\label{lem:3}
    If $q(\widetilde{\z}) = p(\widetilde{\z};\theta)$, then  $\sigma_g^2=0$ and $\zeta_g^2=0$.
\end{lemma}
\vspace*{-0.1in}
\textbf{Remark:} However, the above choice is impractical as sampling from $p(\widetilde{\z};\theta)$ is not easy, and the dependence of $q(\widetilde{\z})$ on $\theta$ would also make our convergence analysis fail.  In practice, we can only hope $q(\widetilde{\z})$ is close to $p(\widetilde{\z}; \theta)$. 

Also, to ensure the partition function $\int p_{0}(\widetilde{\z};\theta) d\widetilde{\z}$ can be written as $\E_{\widetilde{\z} \sim q} \left[ \frac{p_{0}(\widetilde{\z};\theta)}{q(\widetilde{\z})}\right]$, we should guarantee $q(\widetilde{\z}) > 0$ whenever $p_{0}(\widetilde{\z};\theta) > 0$. This is also required in NCE, which is not difficult to satisfy, since many continuous probability distributions can ensure their probability density functions are always positive, e.g., Gaussian, Laplace and  Cauchy distributions. 
Thus, our analysis suggests the following:
\begin{compactenum}
\item Choose $q(\widetilde{\z})$ that can be easily sampled and computed. 
\item We should ensure $q(\widetilde{\z}) > 0$ whenever $p_{0}(\widetilde{\z};\theta) > 0$.
\item The noise distribution should be similar to $p(\widetilde{\z};\theta)$ or the data distribution $p_\text{data}(\widetilde{\z})$. 
\end{compactenum}
To satisfy the last two properties, we can sample the real data and add some noise to the data. In particular,
to sample $\widetilde{\z} \sim q(\widetilde{\z})$, we first sample $\x^{\prime} \sim \D_n$, $\z^{\prime} \sim \mathcal{N}(\mathbf{\mu}, \Sigma)$, and then set $\widetilde{\z} = \x^{\prime} + \z^{\prime} $. Since $\D_n = \{\x_1,\cdots, \x_n\}$ and denote the probability density function of $\mathcal{N}(\cdot; \mathbf{\mu}, \Sigma)$ by $\kappa(\cdot)$, the noise distribution would be $q(\widetilde{\z}) = \frac{1}{n} \sum_{i=1}^{n} \kappa(\x-\x_i)$, which can be approximated within the mini-batch in practice. 
To ensure that the noise is similar to the data, we can also fit the parameter $\mu$ and $\Sigma$ by some training data, which can be easily done by using the existing numpy package via numpy.mean() and numpy.cov(). In our empirical study, this is helpful for high-dimensional problems, i.e., image generation on MNIST data set in Section~\ref{6.3}. For simple problems, such as density estimation and out-of-distribution detection, we can simply set the noise as a multivariate Gaussian distribution, which is very fast and easy to sample from. Since more complex noise would introduce more computations, it is a trade-off in practice.

Finally, we would like to emphasize that the impact of the noise distribution can be alleviated by increasing the mini-batch size for estimating $g(\theta)$ and $\nabla g(\theta)$. If the mini-batch size of noisy samples is $B$, then the variance $\sigma_g^2$ and $\zeta_g^2$ will be scaled by $B$. Hence, the larger the batch size, the less impact of the noise distribution on the convergence rate.

\section{Case Study: Gaussian Mean Estimation}
As pointed out by \citet{liu2022analyzing}, the reason that NCE may fail to learn a good parameter is due to its flat loss landscape. To verify this claim, they use one-dimensional Gaussian mean estimation as an example and show the slow convergence of NCE for this simple task. To demonstrate the effectiveness of our method, we use the same task to show the behavior of our loss and the proposed optimization method. Following their setup, the data and noise distributions are Gaussian distributions with mean $\theta^{*}$ and $\theta_q$ separately, and the variances are both fixed as 1, i.e.,
\begin{align*}
p_\text{data}(x) = \frac{1}{\sqrt{2\pi}}e^{-\frac{(x-\theta^{*})^2}{2}}, \quad
q(x) = \frac{1}{\sqrt{2\pi}}e^{-\frac{(x-\theta_{q})^2}{2}}.
\end{align*}
Then, assume that the model is another Gaussian distribution with mean $\theta$ and variance 1, which is equivalent to setting $p_0(x;\theta)=e^{\theta x - \frac{1}{2}x^2}$. The goal is to learn the parameter $\theta$, or $\tau(\theta)=(\theta, \frac{\theta^2}{2}+\log \sqrt{2\pi})$ for NCE.
First, we can see the flatness of NCE loss via the proposition below.
\begin{prop}\label{prop:1} (Proposition~4.2 in \citet{liu2022analyzing})\\
 Denote that $R=\lvert \theta^{*} - \theta_q \rvert$ and $\J(\cdot)$ is the NCE loss. Then, we have $\J(\tau) - \J(\tau^{*}) \leq R \exp{(-R^2/8)} \Norm{\tau - \tau^{*}}^2$, where $\tau^{*} = \tau(\theta^{*})$ is the optimal solution.
\end{prop}
\textbf{Remark:} If $\theta^{*}$ and $\theta_{q}$ are not close enough, the difference between $J(\tau)$ and $J(\tau^{*})$ will be extremely small when $\tau$ approaches $\tau^{*}$, implying a flat landscape. 

However, it is not a problem for the MLE objective $\LL(\cdot)$ that we optimize, which is shown in the following proposition.
\begin{prop}\label{prop:2}
    For $1$-d Gaussian mean estimation, the MLE objective satisfies that $\LL(\theta) - \LL(\theta^{*}) = \frac{1}{2} \Norm{\theta - \theta^{*}}^2$.
\end{prop}
\textbf{Remark:} Compared with NCE, the MLE objective does not have the extremely small factor $R \exp{(-R^2/8)}$, and thus has a much better loss landscape near the optimum.  

To show this difference more vividly, we plot the loss landscape of the NCE objective and the MLE objective, as well as the newly proposed eNCE objective~\cite{liu2022analyzing}. Following the same setup as \citet{liu2022analyzing}, we set $\theta_q = 0$ and $\theta^{*} =16$. The results are shown in Figure~\ref{fig:loss}. As can be seen, both the NCE and eNCE objectives are very flat near the optimal solution, while the MLE objective is very sharp.

Besides, we provide the convergence rate of our method for this task, as stated below.
\begin{prop} \label{prop:3}
For 1-$d$ Gaussian mean estimation, Algorithm~\ref{alg:multi} ensures $\LL(\theta)-\LL^* \leq \epsilon$ after $T = \O(\epsilon^{-1})$ iterations.
\end{prop}
\textbf{Remark:} Although \citet{liu2022analyzing} can ensure that $\Norm{\tau-\tau^{*}}\leq \epsilon$ after $\O(\epsilon^{-2})$ iterations, by using normalized gradient descent~(NGD) for the NCE or  eNCE objective, their assumptions are very strong. They assume the algorithm can obtain the \textbf{exact} gradient, which is impossible in practice. In contrast, our analysis only relies on stochastic gradients. Also note that NCE optimized with standard gradient descent requires an exponential number of steps to find a reasonable parameter, according to \citet{liu2022analyzing}.

Finally, we conduct experiments to compare different methods.  We choose mean square error~(MSE) $\Norm{\theta - \theta^{*}}^2$ as the criterion and compare our method with NCE trained by SGD and NGD, MCMC training~\cite{NEURIPS2019_378a063b}, and the eNCE objective trained by NGD. The results are presented in Figure~\ref{fig:1}, which demonstrate that our method converges more quickly than other methods, and NCE is the slowest due to its flat loss landscape. Since MLE can be calculated for Gaussian distribution directly, we also include a curve for MLE as a reference, which is very close to the curve of our method after a few steps. Note that we run each method for 100 steps, and we report the running time of each method in Table~\ref{time}, It can be seen that the running time of NCE and our method are very similar, while the MCMC method is much slower. As a result, it is fair to compare the convergence between NCE and our method.

\begin{table}[t]
\caption{Running time of each method.}
\begin{center}
\begin{tabular}{@{}ccccccc@{}}
\toprule
  \text{NCE} &
  \text{NCE~(NGD)}&
  \text{eNCE~(NGD)} &
  \text{MCMC}&
  \text{Ours} \\ \midrule
  \text{176s} &
  \text{181s} &
  \text{169s} &
  \text{1757s} &
  \text{168s}  \\ \bottomrule
\end{tabular}%
\end{center}
\label{time}
\end{table}

\begin{figure}[t]
  	\centering
\includegraphics[width=0.85\columnwidth]{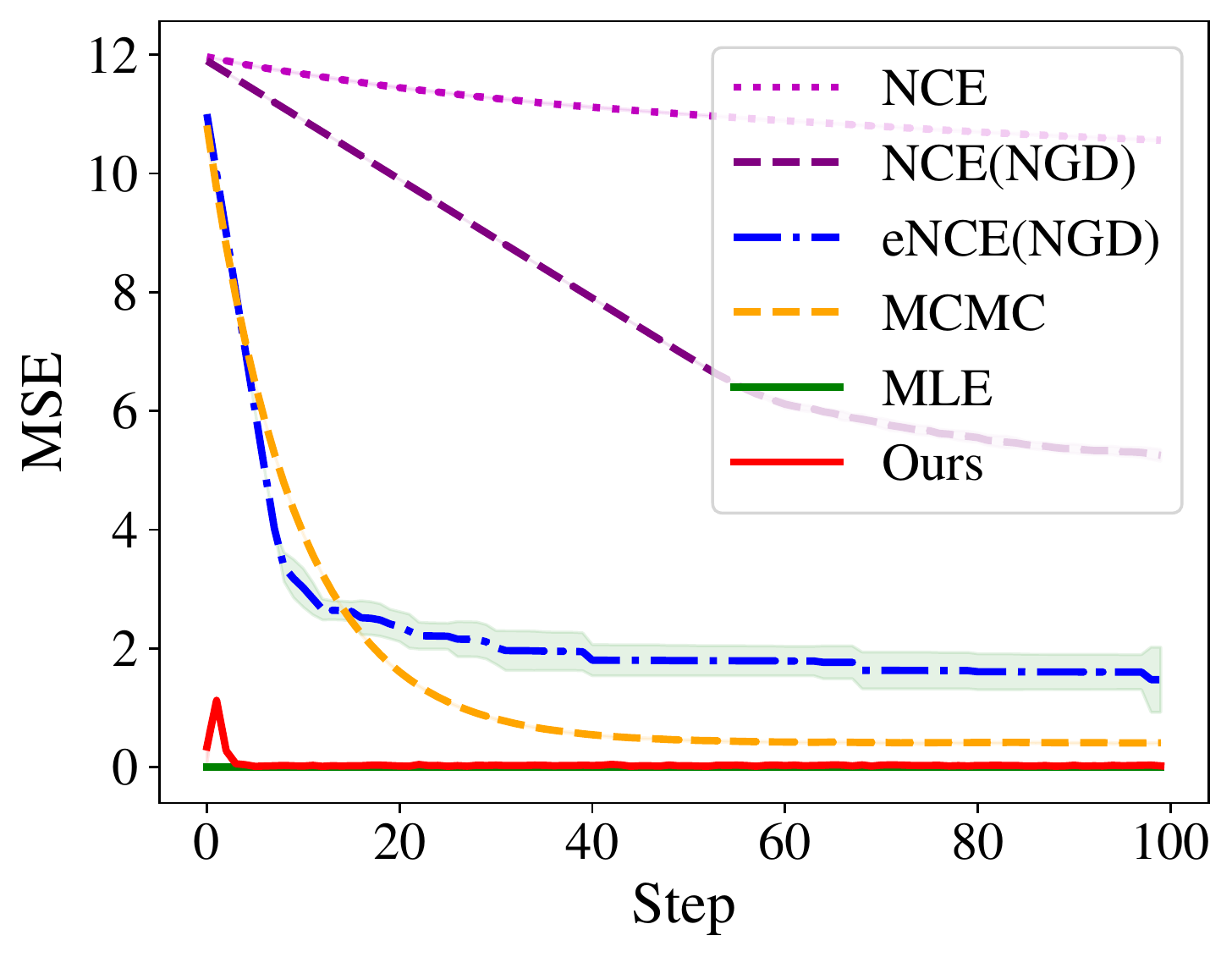}
 	\vskip -0.15in
	\caption{Results for 1-$d$ Gaussian mean estimation.}
 \vskip -0.1in
	\label{fig:1}
\end{figure}

\begin{figure*}[!ht]
\vskip 0.1in
  	\centering
\includegraphics[width=0.75\textwidth]{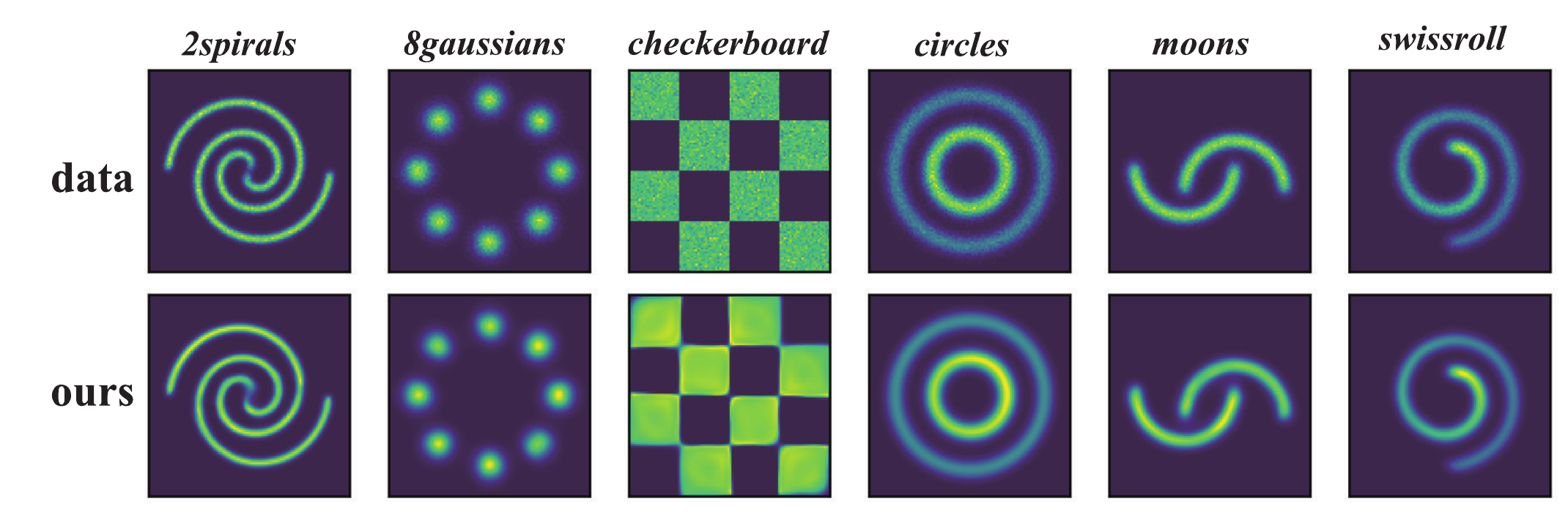}
 	\vskip -0.1in
	\caption{Visualization of density on synthetic datasets.}
	\label{fig:2}
	\vskip -0.1in
\end{figure*}

\begin{table*}[th]
\centering
\caption{Results on synthetic data in terms of MMD~(lower is better).}
\vskip 0.1in
\resizebox{0.95\textwidth}{!}{%
\begin{tabular}{@{}cclcccc@{}}
\toprule
\text{Method} &
  \textit{2spirals} &
  \textit{8gaussians} &
  \textit{checkerboard} &
  \textit{circles} &
  \textit{moons} &
  \textit{swissroll} \\ \midrule
\text{NCE}      & 3.253 ± 0.284 & 0.153 ± 0.095 & 1.956 ± 0.469 & 1.223 ± 0.154 & 5.178 ± 0.341 & 2.715 ± 0.249 \\ \midrule
\text{NCE~(NGD)}  & 3.445 ± 0.287 & 0.177 ± 0.102 & 1.963 ± 0.488 & 1.270 ± 0.292 & 5.010 ± 0.349 & 2.585 ± 0.201 \\ \midrule
\text{eNCE~(NGD)} & 3.328 ± 0.332 & 0.257 ± 0.144 & 1.810 ± 0.218 & 1.183± 0.188 & 4.728 ± 0.399 & 2.975 ± 0.398 \\ \midrule
\text{MCMC} & 3.060\ ±\ 0.780 & 0.150\ ±\ 0.035 & 1.654\ ±\ 0.217 & 1.154\ ±\ 0.294 & 4.722\ ±\ 0.633 & 2.764\ ±\ 0.670 \\ \midrule 
\text{Score Matching} &  3.268\ ±\ 0.846 & 0.250\ ±\ 0.076 & 2.167\ ±\ 0.703 & 1.302\ ±\ 0.272 & 4.826\ ±\ 0.153 & 2.660\ ±\ 0.513\\ \midrule
\text{Contrastive Divergence} &  3.245\ ±\ 0.426 & 0.182\ ±\ 0.085 & 1.987\ ±\ 0.470 & 1.161\ ±\ 0.410 & 4.716\ ±\ 0.658 & 2.623\ ±\ 0.203\\ \midrule 
\textbf{Ours} &
  \textbf{3.040 ± 0.199} &
  \textbf{0.132 ± 0.098} &
  \textbf{1.645 ± 0.251} &
  \textbf{1.075 ± 0.169} &
  \textbf{4.673 ± 0.519} &
  \textbf{2.566 ± 0.274} \\ \bottomrule
\end{tabular}%
}
\vskip -0.1in
\label{tab:my-table}
\end{table*}

\begin{table*}[!th]
\centering
\caption{Results on synthetic data in terms of FID~(lower is better).}
\vskip 0.1in
\resizebox{0.95\textwidth}{!}{%
\begin{tabular}{@{}cclcccc@{}}
\toprule
\text{Method} &
  \textit{2spirals} &
  \textit{8gaussians} &
  \textit{checkerboard} &
  \textit{circles} &
  \textit{moons} &
  \textit{swissroll} \\ \midrule
\text{NCE}      & 0.103 ± 0.038 & 0.118 ± 0.026 & 0.178 ± 0.026 & 0.096 ± 0.018 & 0.114± 0.020 & 0.181 ± 0.041 \\ \midrule
\text{NCE~(NGD)}  & 0.078 ± 0.024 & 0.143 ± 0.048 & 0.169 ± 0.023 & 0.103 ± 0.024 & 0.099 ± 0.029 & 0.171 ± 0.023 \\ \midrule
\text{eNCE~(NGD)} & 0.083 ± 0.040 & 0.128 ± 0.033 & 0.112 ± 0.029 & 0.085 ± 0.045& 0.096 ± 0.015 & 0.212 ± 0.040 \\ \midrule
\text{MCMC} & 0.072\ ±\ 0.052 & 0.115\ ±\ 0.038 & 0.125\ ±\ 0.036 & 0.075\ ±\ 0.042 & 0.109\ ±\ 0.042 & 0.174\ ±\ 0.028 \\ \midrule 
\text{Score Matching} &  0.109\ ±\ 0.044 & 0.178\ ±\ 0.086 & 0.166\ ±\ 0.069 & 0.099\ ±\ 0.026 & 0.117\ ±\ 0.024 & 0.178\ ±\ 0.062\\ \midrule
\text{Contrastive Divergence} &  0.089\ ±\ 0.037 & 0.142\ ±\ 0.051 & 0.121\ ±\ 0.029 & 0.105\ ±\ 0.039& 0.110\ ±\ 0.017 & 0.180\ ±\ 0.025\\ \midrule 
\textbf{Ours} &
  \textbf{0.061 ± 0.032} &
  \textbf{0.104 ± 0.025} &
  \textbf{0.111 ± 0.024} &
  \textbf{0.065 ± 0.030} &
  \textbf{0.094 ± 0.131} &
   \textbf{0.163 ± 0.039} \\ \bottomrule
\end{tabular}%
}
\vskip -0.1in
\label{2tab:my-table}
\end{table*}

\section{Experiments}
In this section, we conduct experiments on three different tasks, and compare our method with NCE, MCMC training,  NCE and eNCE trained by NGD, etc. For our method, we set the parameter $\gamma = 0.1$ and $\beta=0.9$. For MCMC training, the number of sampling steps is searched from the set $\{20, 50, 100\}$ and we use Langevin dynamics~\cite{Langevin} as the sampling approach. For all tasks, we tune the learning rates from $\{1e{-}1, 1e{-}2, 1e{-}3, 1e{-}4\}$ and pick the best one. In this section, all methods are trained with the same training time. Experiments in Section~\ref{6.1} and \ref{6.2} are conducted on a personal laptop, and the training time for each method is around 10 minutes and 72 minutes, respectively. Experiments on MNIST in Section~\ref{6.3} are trained on four NVIDIA Tesla V100 GPUs, and the training time is around 2.8 hours.

\subsection{Density Estimation on Synthetic Data}\label{6.1}
First, we focus on density estimation on synthetic data. Following the experimental setup of the previous literature~\cite{ pmlr-v119-grathwohl20a, Gradient-Guided2022}, we sample a set of 2D data points as the training set, according to some data distribution $p_\text{data}(\x)$, visualized in the top of Figure~\ref{fig:2}. 
Then we train unnormalized models $p_0(\x;\theta)=e^{f_0(\x;\theta)}$ to learn this distribution, where $f_0(\x;\theta)$ is the multi-layer perceptrons (MLPs) with 3 hidden layers and 300 units per layer. In the experiment, We choose the SGD~(or SGD-style) optimizer. For NCE, eNCE and our method, the noise distribution is selected as a multivariate Gaussian distribution, whose mean and variance are fitted on the training set. 

To quantify the performance of different methods, we adopt the maximum mean discrepancy~(MMD)~\cite{MMD} as the criterion. The MMD metric is widely used to compare different distributions, and a lower MMD indicates that the two distributions are more similar. We sample 10000 points from the data distribution and the learned model, and the computed MMD metric is shown in Table~\ref{tab:my-table}. As can be seen, our method enjoys the lowest MMD among all methods for all six cases, indicating the superiority of the proposed method. Also, we report the Frechet Inception Distance~(FID) score of each method, which is another widely-used measure in comparing distributions. The results are presented in Table~\ref{2tab:my-table}, and our method enjoys better FID scores than other algorithms~(smaller is better). Besides, we compare the estimated density and the ground-truth in Figure~\ref{fig:2}, showing that our method learns the density accurately in most cases. Finally, We compare different methods with the Adam optimizer and investigate the behavior of our algorithm with different values of $\gamma$ and $\beta$ in the appendix. 

\begin{table*}[ht]
\caption{OOD detection results (AUROC$\uparrow$, AUPRC$\uparrow$ and FRP80$\downarrow$) for models trained on CIFAR-10.}
\centering
\resizebox{0.9\textwidth}{!}{%
\begin{tabular}{|c|cl|}
\hline
\text{Dataset}       & \multicolumn{2}{l|}{\qquad \qquad \qquad \qquad \qquad \qquad \ 
 AUROC↑ \qquad \qquad \qquad \ \ AUPRC↑ \qquad \qquad \qquad \ \ \ \ FRP80↓} \\ \hline
                                & NCE       & \ \ 0.6468\ ±\ 0.0122  \qquad \quad \ \ \ 0.5455\ ±\ 0.0053  \qquad \quad \ \ \ 0.4929\ ±\ 0.0301           \\
                                & NCE~(NGD)   & \ \ 0.6748\ ±\ 0.0187 \qquad \quad \ \ \  0.6265\ ±\ 0.1565 \qquad \quad \ \  \ 0.5392\ ±\ 0.0271           \\
                                & eNCE~(NGD)  &  \ \ 0.7451\ ±\ 0.0296 \qquad \quad \ \ \ 0.7052\ ±\ 0.0310\qquad \quad \ \  \ \ 0.4052\ ±\ 0.0586                                \\
                                & MCMC       & \ \ 0.7077\ ±\ 0.0154  \qquad \quad \ \ \ 0.6787\ ±\ 0.0119  \qquad \quad \ \ \ 0.5137\ ±\ 0.0253           \\
                                 & Scoring Matching       & \ \ 0.5329\ ±\ 0.0024  \qquad \quad \ \ \ 0.5304\ ±\ 0.0029  \qquad \quad \ \ \ 0.7706\ ±\ 0.0047           \\
                                  & Contrastive Divergence       & \ \ 0.7735\ ±\ 0.0161  \qquad \quad \ \ \ 0.7371\ ±\ 0.0106  \qquad \quad \ \ \ 0.3733\ ±\ 0.0382           \\
\multirow{-7}{*}{\text{CIFAR10-Interp}} & \textbf{Ours} &  \ \ \textbf{0.8019\ ±\ 0.0323} \qquad \quad \ \ \  \textbf{0.7679\ ±\ 0.0402} \qquad \quad \ \ \  \textbf{0.3185\ ±\ 0.0621} \\ \hline
                                & NCE       & \ \ 0.6425\ ±\ 0.0107 \qquad \quad \ \ \  0.3436\ ±\ 0.0052 \qquad \quad \ \  \ 0.5592\ ±\ 0.0286           \\
                                & NCE~(NGD)   & \ \ 0.6593\ ±\ 0.0035 \qquad \quad \ \ \  0.4361\ ±\ 0.0083 \qquad \quad \ \  \ 0.7284\ ±\ 0.0040           \\
                                & eNCE~(NGD)  &  \ \ 0.6612\ ±\ 0.0075 \qquad \quad \ \ \ 0.4633\ ±\ 0.0109\qquad \quad \ \  \ \ 0.6910\ ±\ 0.0134                                   \\
                                & MCMC       & \ \ 0.5359\ ±\ 0.0078  \qquad \quad \ \ \ 0.2674\ ±\ 0.0048 \qquad \quad \ \ \ 0.6327\ ±\ 0.0037           \\
                                 & Scoring Matching       & \ \ 0.5601\ ±\ 0.0068  \qquad \quad \ \ \ 0.3237\ ±\ 0.0010  \qquad \quad \ \ \ 0.8346\ ±\ 0.0060           \\
                                  & Contrastive Divergence       & \ \ 0.6103\ ±\ 0.0040  \qquad \quad \ \ \ 0.3057\ ±\ 0.0018  \qquad \quad \ \ \ 0.3758\ ±\ 0.0025           \\
\multirow{-7}{*}{\text{SVHN}} & \textbf{Ours} &  \ \ \textbf{0.7843\ ±\ 0.0057} \qquad \quad \ \ \  \textbf{0.5134\ ±\ 0.0076} \qquad \quad \  \ \ \textbf{0.3255\ ±\ 0.2081} \\ \hline
                                & NCE       & \ \ 0.5323\ ±\ 0.0199\qquad \quad \ \ \  \ 0.5129\ ±\ 0.0095 \qquad \quad \ \ \ 0.7122\ ±\ 0.0216             \\
                                & NCE~(NGD)   &\ \ 0.5513\ ±\ 0.0218\qquad \quad \ \ \  \ 0.5458\ ±\ 0.0160\qquad \quad \ \  \ \ 0.7832\ ±\ 0.0306             \\
                                & eNCE~(NGD)  & \ \ 0.5034\ ±\ 0.0069 \qquad \quad \ \ \ 0.5248\ ±\ 0.0158\qquad \quad \ \  \ \ 0.8263\ ±\ 0.0095                                \\
                                & MCMC       & \ \ 0.5669\ ±\ 0.0059  \qquad \quad \ \ \ 0.5604\ ±\ 0.0138  \qquad \quad \ \ \ 0.7217\ ±\ 0.0052           \\
                                 & Scoring Matching       & \ \ 0.5732\ ±\ 0.0014  \qquad \quad \ \ \ 0.5343\ ±\ 0.0070  \qquad \quad \ \ \ 0.6943\ ±\ 0.0059           \\
                                  & Contrastive Divergence       & \ \ 0.5276\ ±\ 0.0098  \qquad \quad \ \ \ 0.5136\ ±\ 0.0120  \qquad \quad \ \ \ 0.6642\ ±\ 0.0111           \\
\multirow{-7}{*}{\text{CIFAR-100}}      & \textbf{Ours} &  \ \ \textbf{0.6044\ ±\ 0.0049} \qquad \quad \ \  \ \textbf{0.6262\ ±\ 0.0079} \qquad \quad \  \ \ \textbf{0.5795\ ±\ 0.0113} \\ \hline
                                & NCE       &\ \ 0.5356\ ±\ 0.0006 \qquad \quad \ \ \ 0.4995\ ±\ 0.0005 \qquad \quad \ \ \ 0.6876\ ±\ 0.0028             \\
                                & NCE~(NGD)   &\ \ 0.6049\ ±\ 0.0189  \qquad \quad \ \ \ 0.5817\ ±\ 0.0083\qquad \quad \ \ \ \   0.7119\ ±\ 0.0408          \\
                                & eNCE~(NGD)  &  \ \ 0.5470\ ±\ 0.0039 \qquad \quad \ \ \ \textbf{0.6840\ ±\ 0.0042}\qquad \quad \ \  \ \ 0.9450\ ±\ 0.0028                                \\
                                & MCMC       & \ \ 0.5359\ ±\ 0.0153  \qquad \quad \ \ \ 0.5107\ ±\ 0.0157  \qquad \quad \ \ \ 0.7436\ ±\ 0.0131           \\
                                 & Scoring Matching       & \ \ 0.5419\ ±\ 0.0057  \qquad \quad \ \ \ 0.5221\ ±\ 0.0030  \qquad \quad \ \ \ 0.7372\ ±\ 0.0097           \\
                                  & Contrastive Divergence       & \ \ 0.5044\ ±\ 0.0054  \qquad \quad \ \ \ 0.4980\ ±\ 0.0060  \qquad \quad \ \ \ 0.6248\ ±\ 0.0022           \\
\multirow{-7}{*}{\text{LSUN-C}} &\textbf{Ours} &  \ \ \textbf{0.6944\ ±\ 0.0061} \qquad \quad \ \  \ {0.6267\ ±\ 0.0046} \qquad \quad \ \ \  \textbf{0.5198\ ±\ 0.0126} \\ \hline
\end{tabular}%
}
\vskip -0.1in
\label{tab:4}
\end{table*}

\subsection{Out-of-distribution Detection} \label{6.2}
Then, we experiment on the out-of-distribution~(OOD) detection, since OOD detection performance is an important measure of the density estimation quality. For this task, we choose CIFAR-10~\cite{Krizhevsky2009Cifar10} as the in-distribution data. We use the energy-based model as our unnormalized model, by setting $p_0(\x;\theta) = e^{f_0(\x;\theta)}$, where $f_0(\x;\theta)$ is a 40-layer WideResNet~\cite{WideRes}. The noise distribution for NCE, eNCE and our method is selected as the multivariate Gaussian distribution, and we use Adam to optimize.  For the OOD test dataset, we use four common benchmarks: CIFAR-10 Interp, SVHN~\cite{Netzer2011SVHN}, CIFAR-100~\cite{Krizhevsky2009Cifar10}, and LSUN~\cite{Yu2015}. We decide whether a test sample $\x$ is anomalous or not by computing the  density $p_0(\x;\theta)$, where a higher value indicates the test sample is more likely to be a normal sample. 

To measure the performance of different methods, we follow previous works~\cite{NEURIPS2019_1e795968, pmlr-v139-havtorn21a,geng2021bounds} to choose three metrics to compare:  (1) area under the receiver operating characteristic curve (AUROC$\uparrow$); (2) area under the precision-recall curve (AUPRC$\uparrow$); and (3) false positive rate at 80\% true positive rate~(FPR80$\downarrow$), where the arrow indicates the direction of improvement of the metrics. These three metrics are commonly used for evaluating OOD detection methods, and the results are reported in Table~\ref{tab:4}. As can be seen, our method performs the best under three different criteria in most cases, implying the effectiveness of the proposed method.

\begin{table*}[ht]
\caption{Average negative log-likelihood~(smaller is better).}
\begin{center}
\resizebox{0.85\textwidth}{!}{%
\begin{tabular}{@{}ccccccc@{}}
\toprule
  \text{NCE} &
  \text{NCE~(NGD)}&
  \text{eNCE~(NGD)} &
  \text{MCMC-OT}&
  \text{CF-EBM} &
  \text{CoopNets} &
  \text{Ours}\\ \midrule
  \text{1.457 ± 0.003} &
  \text{1.618 ± 0.010} &
  \text{1.600 ± 0.044} &
  \text{1.441 ± 0.012} &
  \text{1.864 ± 0.004} &
  \text{1.591 ± 0.022} &
  \textbf{1.394 ± 0.007} \\ \bottomrule
\end{tabular}%
}    
\end{center}
\vskip -0.15in
\label{tab:2}
\end{table*}

\subsection{Learning on Real Image Dataset}\label{6.3}
Finally, we test our method on two real image datasets, MNIST~\cite{LeCun1998MNIST} and CIFAR-10~\cite{Krizhevsky2009Cifar10}. We describe the setup and the results of MNIST in this subsection, and results of CIFAR-10 can be found in Appendix~\ref{a1}.
For MNIST task, following the same setup in TRE~\cite{Rhodes2020}, the model takes the form of $p_0(\x;\theta) = e^{f_0(\x;\theta)}$, where $f_0(\x;\theta)$ is the 18-layer ResNet~\cite{Resnet18}. For NCE, eNCE and our method, we try three different noise distributions: 1) the empirical distribution $\D_n = \{\x_1,\cdots,\x_n \}$; 2)  multivariate Gaussian distribution, with mean and covariance fitted by the training data; 3) mixture of distribution $\D_n$ and a  multivariate Gaussian distribution. We find the last one is the best choice for all methods. We use Adam to optimize NCE, and generate new samples via MCMC sampling after training the model. The generated samples of our method with different noises are shown in the Appendix~\ref{a1}, and the result with the third noise distribution is presented in Figure~\ref{fig:MNIST}. 
 
We also evaluate the learned model via estimated average negative log-likelihood (bits per dimension) on the
testing sets in Table~\ref{tab:2}, which is computed by the annealed importance sampling~(AIS)~\cite{pmlr-v97-durkan19a}. We also compare our method with some other methods, including MCMC-OT~\cite{add1}, CF-EBM~\cite{add2}, and CoopNets~\cite{add3}. As can be seen, our method attains a better likelihood than other methods, indicating the proposed method performs better in this task. 

\begin{figure}[ht]
\vskip -0.2in
\begin{center}
\includegraphics[width=0.85\columnwidth]{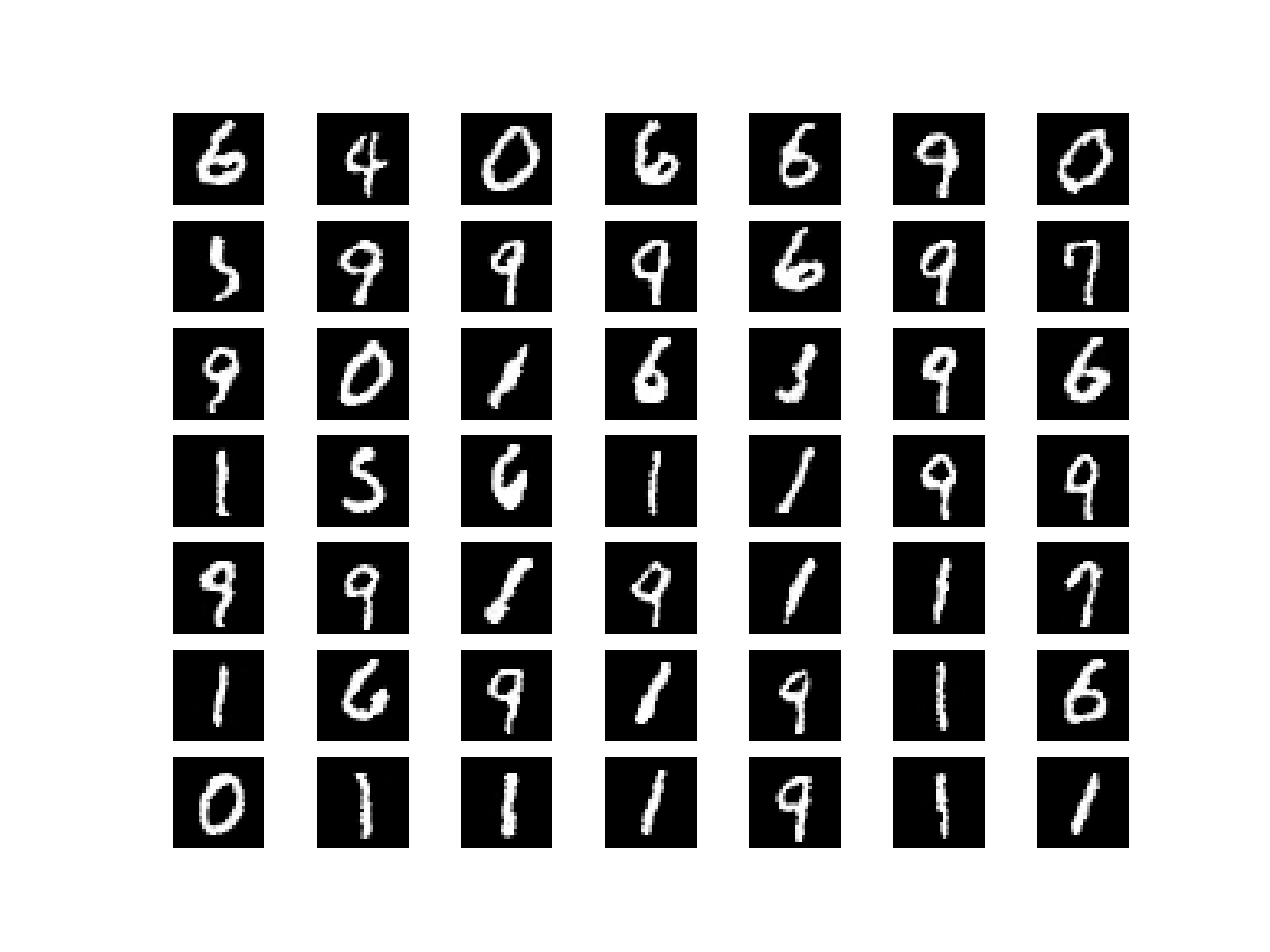}
\vskip -0.25in
\caption{Generated digits using our model via MCMC sampling.}
\label{fig:MNIST}
\end{center}
\vspace{-0.15in}
\end{figure}

\section{Conclusion}
In this paper, we investigate the problem of learning unnormalized models by maximum likelihood estimation. By introducing a noise distribution, we cast the problem as compositional optimization and utilize a stochastic algorithm to solve it. We provide the convergence rate of our method and analyze its relationship with the noise distribution. Besides, we use one-dimensional Gaussian mean estimation as an example to show the better loss landscape of our loss compared with the NCE loss, and the fast convergence of the proposed method. Finally, experiments on practical problems demonstrate the effectiveness of the proposed method.  

\section{Acknowledgments}
W. Jiang, L. Wu and L. Zhang are partially supported by the National Key R\&D Program of China (2021ZD0112802), NSFC (62122037, 61921006), and the Fundamental Research Funds for the Central Universities (2023300246).

\newpage
\bibliography{ref}
\bibliographystyle{icml2023}

\newpage
\appendix
\onecolumn

\section{Omitted Experimental Results}\label{a1}
In this section, we present omitted experimental results on the task of density estimation and image generation. 
\subsection{More Results on Density Estimation}
\paragraph{Results for Adam optimizer}
First, we show the performance of different methods with Adam~(or Adam-style) optimizer in the task of Density Estimation on Synthetic Data. The results are shown in Table~\ref{3tab:my-table} and Table~\ref{4tab:my-table}. As can be seen, our method performs better than other algorithms in most cases.
\begin{table*}[ht]
\centering
\vskip -0.2in
\caption{Results on synthetic data in terms of MMD with Adam optimizer~(lower is better).}
\resizebox{0.9\textwidth}{!}{%
\begin{tabular}{@{}cclcccc@{}}
\toprule
\text{Method} &
  \textit{2spirals} &
  \textit{8gaussians} &
  \textit{checkerboard} &
  \textit{circles} &
  \textit{moons} &
  \textit{swissroll} \\ \midrule
\text{NCE}      & 3.230 ± 0.460 & 0.130 ± 0.055 & 1.710 ± 0.190 & 1.045 ± 0.208 & 4.588 ± 0.688 & 1.510 ± 0.498 \\ \midrule
\text{NCE~(NGD)}  & 2.533 ± 0.716 & 0.128 ± 0.020 & 1.807 ± 0.451 & 0.947 ± 0.265 & 4.202 ± 0.474 & 1.576 ± 0.589 \\ \midrule
\text{eNCE~(NGD)} & 2.585 ± 0.841 & 0.146 ± 0.082 & 1.690 ± 0.264 & 0.968 ± 0.342 & 3.852 ± 0.957 & 1.604 ± 0.490 \\ \midrule
\text{MCMC} & \textbf{2.342\ ±\ 0.658} & 0.178\ ±\ 0.105 & 1.710\ ±\ 0.394 & 0.937\ ±\ 0.387 & 3.680\ ±\ 0.916 & 1.438\ ±\ 0.375 \\ \midrule 
\text{Score Matching} &  2.727 ±\ 0.579 & 0.135\ ±\ 0.071 & 3.130\ ±\ 0.497 & 1.706\ ±\ 0.309 & 5.296\ ±\ 0.221 & 2.874\ ±\ 0.688\\ \midrule
\text{Contrastive Divergence} &  2.516\ ±\ 0.984 & 0.173\ ±\ 0.066 & 1.683\ ±\ 0.904 & 0.927\ ±\ 0.432 & 3.904\ ±\ 0.914 & 1.498\ ±\ 0.368\\ \midrule 
\textbf{Ours} &
  2.488 ± 0.356 &
  \textbf{0.117 ± 0.086} &
  \textbf{1.651 ± 0.166} &
  \textbf{0.886 ± 0.129} &
  \textbf{3.644 ± 0.490} &
  \textbf{1.338 ± 0.131} \\ \bottomrule
\end{tabular}%
}
\label{3tab:my-table}
\end{table*}

\begin{table*}[h]
\vskip -0.2in
\centering
\caption{Results on synthetic data in terms of FID with Adam optimizer~(lower is better).}
\resizebox{0.9\textwidth}{!}{%
\begin{tabular}{@{}cclcccc@{}}
\toprule
\text{Method} &
  \textit{2spirals} &
  \textit{8gaussians} &
  \textit{checkerboard} &
  \textit{circles} &
  \textit{moons} &
  \textit{swissroll} \\ \midrule
\text{NCE}      & 0.053 ± 0.030 & 0.096 ± 0.042 & 0.156 ± 0.121 & 0.068 ± 0.028 & 0.046 ± 0.031 & 0.159 ± 0.069 \\ \midrule
\text{NCE~(NGD)}  & 0.067 ± 0.315 & 0.104 ± 0.038 & 0.167 ± 0.091 & 0.061 ± 0.030 & 0.029 ± 0.017 & 0.171 ± 0.065 \\ \midrule
\text{eNCE~(NGD)} & 0.072 ± 0.053 & 0.127 ± 0.085 & 0.102 ± 0.081 & 0.066 ± 0.046 & 0.038 ± 0.032 & 0.205 ± 0.035 \\ \midrule
\text{MCMC} & 0.047\ ±\ 0.025 & 0.107\ ±\ 0.073 & 0.097\ ±\ 0.084 & 0.049\ ±\ 0.028 & 0.052\ ±\ 0.021 & 0.161\ ±\ 0.088 \\ \midrule 
\text{Score Matching} &  0.070\ ±\ 0.024 & 0.087\ ±\ 0.050 & 0.231\ ±\ 0.110 & 0.086\ ±\ 0.034 & 0.037\ ±\ 0.018 & 0.190\ ±\ 0.148\\ \midrule
\text{Contrastive Divergence} &  0.064\ ±\ 0.025 & 0.092\ ±\ 0.034 & 0.093\ ±\ 0.034 & 0.076\ ±\ 0.038 & 0.062\ ±\ 0.022 & 0.148\ ±\ 0.029\\ \midrule 
\textbf{Ours} &
  \textbf{0.040 ± 0.026} &
  \textbf{0.067 ± 0.030} &
  \textbf{0.086 ± 0.101} &
  \textbf{0.040 ± 0.035} &
  \textbf{0.028 ± 0.037} &
  \textbf{0.144 ± 0.078}  \\ \bottomrule
\end{tabular}%
}
\vskip -0.1in
\label{4tab:my-table}
\end{table*}

\paragraph{Results with different $\gamma$ and $\beta$}
Then, we investigate the behavior of our algorithm  with different values of $\gamma$ and $\beta$. In the previous experiment, we simply set $\gamma=0.1$ and $\beta=0.9$. Here, we first fix  $\gamma=0.1$ and enumerate $\beta$ from the set $\{0.8, 0.9, 0.99\}$. Then, we fix $\beta=0.9$ and enumerate $\gamma$ from the set $\{0.01, 0.1, 0.2\}$. The results are reported in Table~\ref{6tab:my-table} and Table~\ref{7tab:my-table}, which indicates that our method is not very sensitive to the choice of $\beta$ and $\gamma$ within a certain range.
\begin{table*}[ht]
\vskip -0.1in
\centering
\caption{Results on synthetic data in terms of MMD ~(lower is better).}
\resizebox{0.9\textwidth}{!}{%
\begin{tabular}{@{}cclcccc@{}}
\toprule
\text{value} &
  \textit{2spirals} &
  \textit{8gaussians} &
  \textit{checkerboard} &
  \textit{circles} &
  \textit{moons} &
  \textit{swissroll} \\ \midrule
\text{$\beta =0.8$}      & 3.355\ ±\ 0.781&0.168\ ±\ 0.098	&1.889\ ±\ 0.506	&1.355\ ±\ 0.164	&4.762\ ±\ 0.192	&2.784\ ±\ 0.618 \\ \midrule
\text{$\beta =0.9$}  & 3.040\ ±\ 0.199	&0.132\ ±\ 0.098	&1.645\ ±\ 0.251&	1.075\ ±\ 0.169	&4.673\ ±\ 0.519&	2.566\ ±\ 0.274 \\ \midrule
\text{$\beta =0.99$} & 3.038\ ±\ 0.395&	0.150\ ±\ 0.078&	1.603\ ±\ 0.258&	1.120\ ±\ 0.211&	4.816\ ±\ 0.292&	2.528\ ±\ 0.371
  \\ \bottomrule
\end{tabular}%
}
\label{6tab:my-table}
\end{table*}

\begin{table*}[!htbp]
\centering
\vskip -0.2in
\caption{Results on synthetic data in terms of MMD~(lower is better).}
\resizebox{0.9\textwidth}{!}{%
\begin{tabular}{@{}cclcccc@{}}
\toprule
\text{Method} &
  \textit{2spirals} &
  \textit{8gaussians} &
  \textit{checkerboard} &
  \textit{circles} &
  \textit{moons} &
  \textit{swissroll} \\ \midrule
\text{$\gamma=0.01$}      & 	2.982\ ±\ 0.304&	0.141\ ±\ 0.106	&1.730\ ±\ 0.716	&1.345\ ±\ 0.198	&4.650\ ±\ 0.302&	2.595\ ±\ 0.493 \\ \midrule
\text{$\gamma=0.1$}  & 3.040
\ ±\
 0.199	&0.132\ ±\ 0.098	&1.645\ ±\ 0.251&	1.075\ ±\ 0.169&	4.673\ ±\ 0.519	&2.566\ ±\ 0.274\\ \midrule
\text{$\gamma=0.2$} & 	3.168\ ±\ 0.360&	0.145\ ±\ 0.052	&1.718\ ±\ 0.418&	1.152\ ±\ 0.109&	4.772\ ±\ 0.651	&2.648\ ±\ 0.601
 \\ \bottomrule
\end{tabular}%
}
\label{7tab:my-table}
\end{table*}

\subsection{More Results on Image Generation}
\paragraph{Training EBM on MNIST}
To verify our analysis for the noise distribution, we compare the image generation results with different noises: 1) the empirical distribution $\D_n = \{\x_1,\cdots,\x_n \}$; 2) Gaussian distribution with its mean and variance fitted on the training data; 3) mixture of empirical distribution $\D_n$ and a fitted Gaussian noise. The results are shown in Figure~\ref{fig:3}. As can be seen, the first two cases perform poorly and only the third case generates images similar to the original MNIST dataset, which is consistent with our suggestions.
\begin{figure*}[ht]
\vskip -0.1in
\begin{center}
\centerline{\subfigure[empirical distribution]{\includegraphics[width=0.35\textwidth]{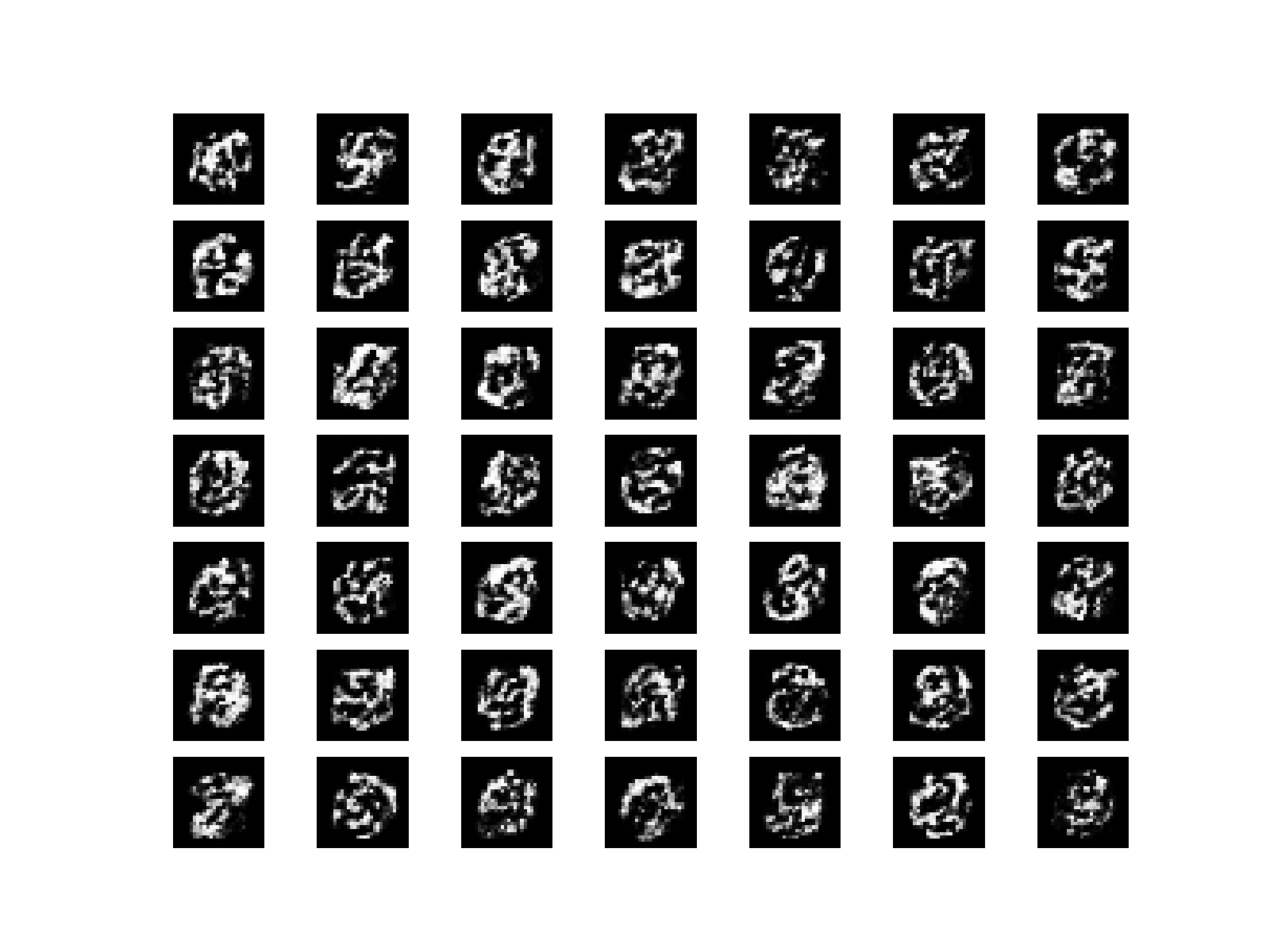}}
\hspace{-0.3in}
\subfigure[fitted Gaussian distribution]
{\includegraphics[width=0.35\textwidth]{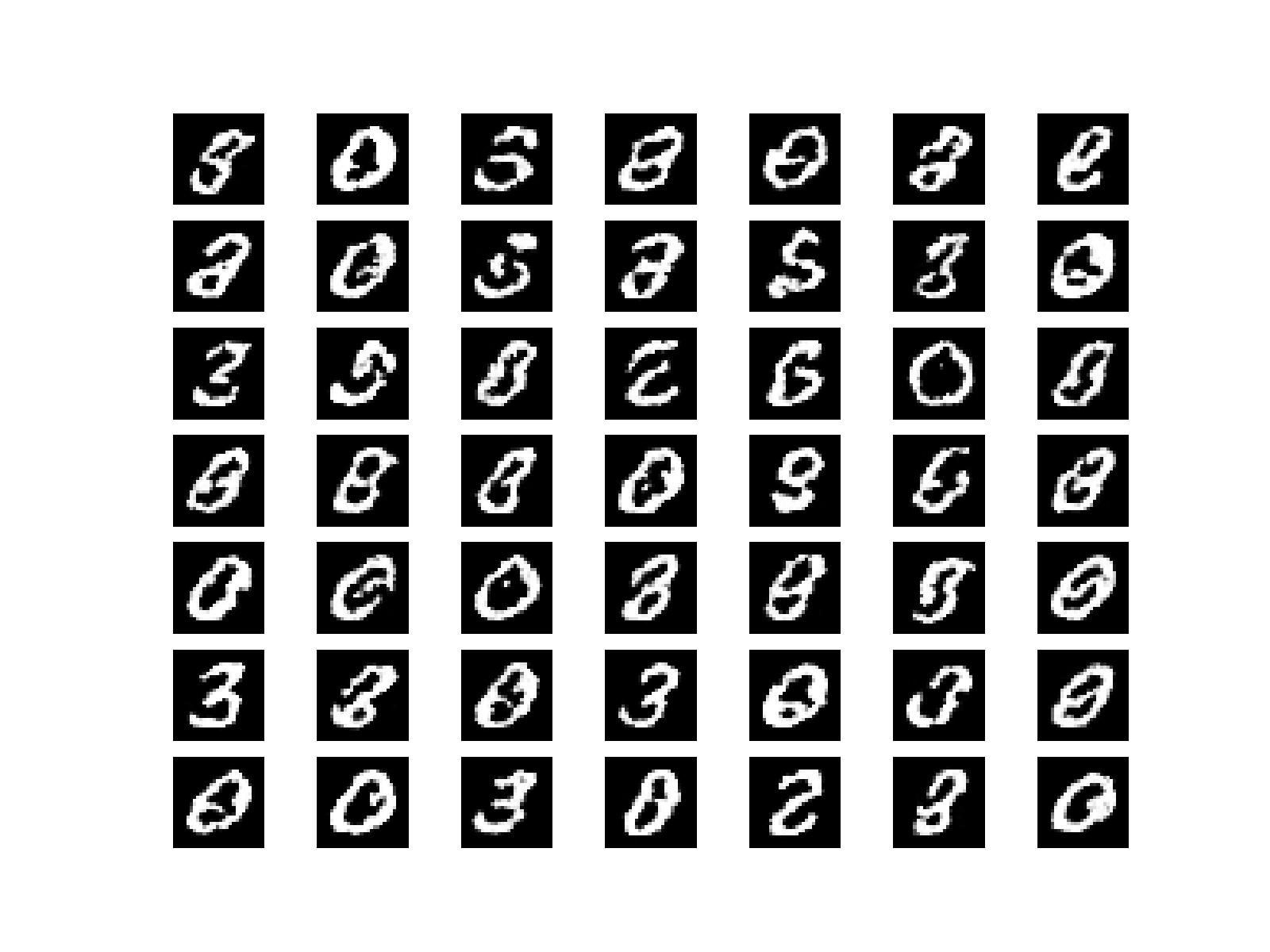}}
\hspace{-0.3in}
\subfigure[empirical distribution + noise]{\includegraphics[width=0.35\textwidth]{./Fig/M.pdf}}}
\vskip -0.1in
\caption{Generated samples with different noise distributions.}
\label{fig:3}
\end{center}
\vskip -0.3in
\end{figure*}
\paragraph{Training EBM on CIFAR-10}
We further test our method on the more complex CIFAR-10 dataset. We use the codebase of the JEM model \cite{Grathwohl2020Your}, which performs  generation and classification simultaneously. We replace the standard EBM with Langevin dynamics with our proposed method, and follow exactly the same parameter setting as specified in the codebase \cite{jem_code}. We test our method for both classification and image generation tasks, and compare with the original JEM method. For this sets of experiments on more complex natural images, we find that it is more important to design an appropriate noise distribution. Our {\em empirical distribution + noise} method does not work well in this case, partially due to the fact that the empirical data are too sparsely located at the complex data manifold. Consequently, the estimated density can induce too large variance, leading to slow convergence. We mitigate the limitation by introducing a few MCMC steps (around 10) by directly starting from a pool of updating negative samples, and adopt the kernel density estimation to approximate the density values (as they do not have closed forms after MCMC steps). We find this is essential for learning to generate high-quality real images. Note our adopted method still has limitations, for example, it still requires MCMC steps that we want to avoid, and the estimated densities are not accurate enough. We thus leave designing a good noise estimation for more advanced image generation as interesting future work. Following the instruction in the codebase, we  achieve a classification accuracy of 90.72\% with a loss of 0.34, compared to an accuracy of 89.38\% and a loss of 0.39 from the JEM. Regarding the generated image quality, we measure it with the Inception Score (IS) and Fréchet Inception Distance~(FID) score \cite{Seitzer2020FID}, based on 10k generated images. Our method obtains an improved FID score from 54.76 to 51.77, while maintaining comparable IS scores, {\it i.e.}, 4.79 (ours) versus 4.83 (JEM). Note that these numbers are a little worse than what were reported in the JEM paper, mostly due to different experimental settings\footnote{As confessed by the JEM author, their reported results need heavily manual tuning, which is not easily reproduced (see issue \#7 in the codebase). The results are also far from the current state-of-the-arts. However, our goal is only to demonstrate the effectiveness of our  method compared to the standard EBM training under a  simple setting. Thus, our results are reasonable and the comparison is fair. }. Some generated samples from our method are plotted in Figure~\ref{fig:cifar}.
\begin{figure}[!htbp]
\vskip 0.1in
  	\centering
\includegraphics[width=0.35\linewidth]{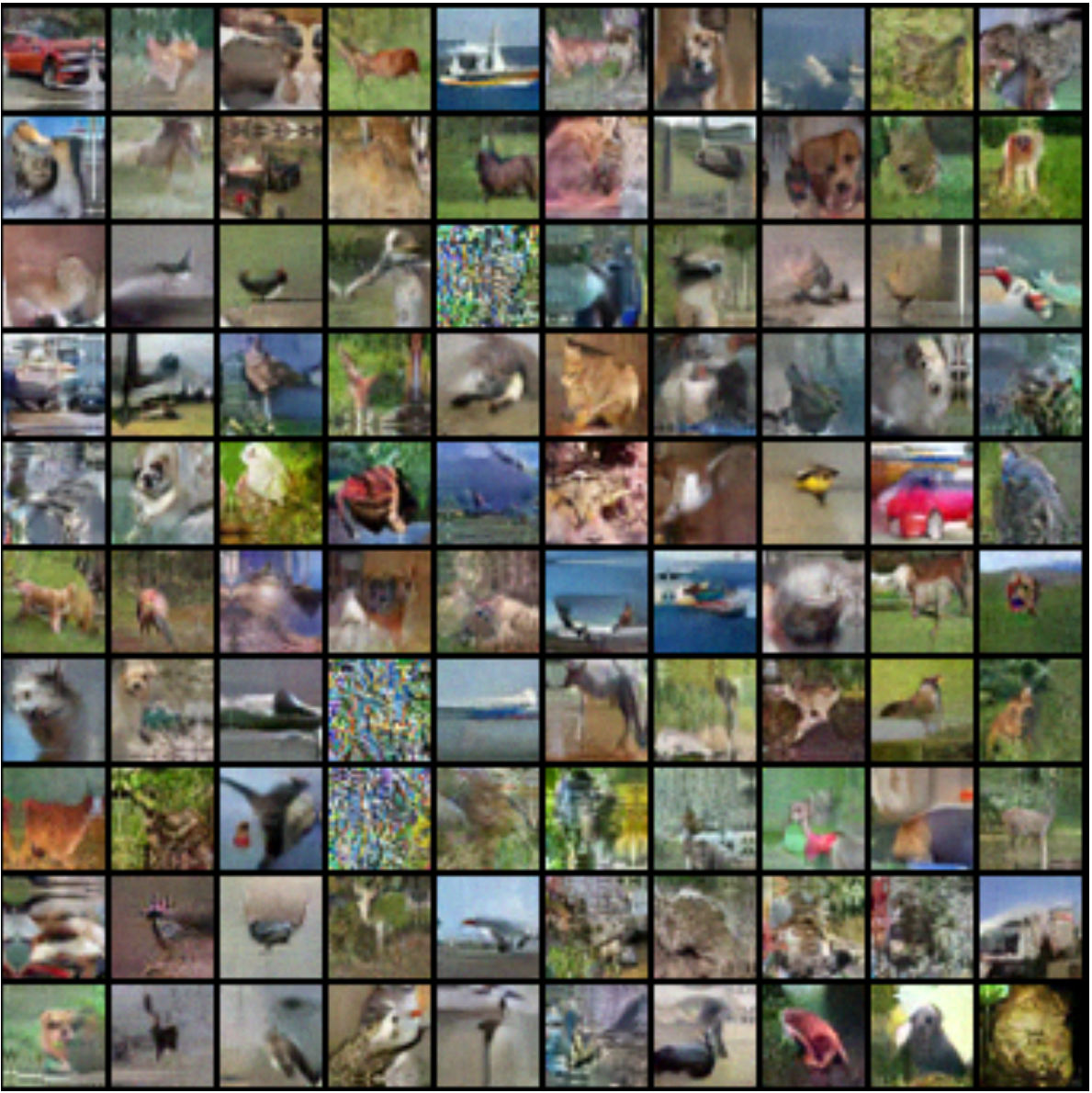}
 	\vskip -0.1in
\caption{Generated samples on the CIFAR-10 dataset. }
	\label{fig:cifar}
	\vskip -0.1in
\end{figure}
\section{Analysis}
In this section, we present the proof of all theorems.
\subsection{Proof of Theorem~\ref{thm:1}}
Here, we prove the first property of Theorem~\ref{thm:1}. The other properties ---  asymptotically normal and  asymptotically optimal can be obtained according to the property of MLE~(for example, see Chapter 9 in ~\citet{Wasserman}).


We note that the minimizing of $\LL(\theta) = -\frac{1}{n}\sum_{i=1}^n \log p(\x_i;\theta)$ is equivalent to minimizing
\begin{align*}
M_{n}(\theta) = -\frac{1}{n}\sum_{i=1}^n \log \frac{p(\x_i;\theta)}{p(\x_i;\theta^{*})} = \frac{1}{n}\sum_{i=1}^n \log \frac{p(\x_i;\theta^{*})}{p(\x_i;\theta)}.
\end{align*}

By the law of large numbers, as the sample size increases, $M_{n}(\theta)$ converges to
\begin{align*}
M(\theta) = \E_{\x \sim p_\text{data}}\left[ \log \frac{p(\x;\theta^{*})}{p(\x;\theta)} \right] = \int p(\x;\theta^{*})  \log \frac{p(\x;\theta^{*})}{p(\x;\theta)}   d\x  =D(\theta^{*},\theta),
\end{align*}
where $D(\theta^{*},\theta)$ means the Kullback-Leibler distance between $p(\x;\theta^{*})$ and $p(\x;\theta)$. Hence, we have $M_n({\theta}) \stackrel{P}{\rightarrow} D(\theta^{*},\theta)$. 

Then, we prove the first property --- consistent is satisfied under the following conditions.
\begin{align*}
1) \sup_{\theta} \left|M_n(\theta)-M(\theta)\right| \stackrel{P}{\rightarrow} 0; \quad \quad 2) \sup _{\theta:\left|\theta-\theta^{*}\right| \geq \epsilon} M(\theta^{*})<M\left(\theta\right), \text{ for every } \epsilon>0.
\end{align*}

Since $\widehat{\theta}$ minimizes $M_n(\theta)$, we have $M_n(\widehat{\theta}) \leq M_n(\theta^{*})$ and 
\begin{align*}
M(\widehat{\theta}) -  M(\theta^{*}) &= M_n(\theta^{*})  -  M(\theta^{*}) + M(\widehat{\theta}) - M_n(\theta^{*}) \\
&\leq M_n(\theta^{*})  -  M(\theta^{*}) + M(\widehat{\theta}) - M_n(\widehat{\theta}) \\
&\leq M_n(\theta^{*})  -  M(\theta^{*}) + \sup_{\theta}\left| M(\theta) - M_n(\theta)  \right| \\
&\stackrel{P}{\rightarrow} 0.
\end{align*}
It follows that, for any $\delta > 0$,
\begin{align*}
\P\left(M(\theta^{*}) < M(\widehat{\theta}) - \delta \right) \rightarrow 0.
\end{align*}
Select any $\epsilon > 0$. By condition 2), there exists $\delta > 0$ such that $\left|\theta-\theta^{*} \right|\geq \epsilon$ implies that $M(\theta^{*}) < M(\widehat{\theta}) -\delta$. Hence,
\begin{align*}
\P\left(\left| \widehat{\theta} - \theta^{*} \right| > \epsilon \right) \leq \P\left(M(\theta^{*}) < M(\widehat{\theta}) - \delta \right) \rightarrow 0.
\end{align*}

\subsection{Proof of Theorem~2}

First, we prove the following supporting lemmas. We simplify the notation $g(\theta; \widetilde{\z})$ as $\widetilde{g}(\theta_{t})$ and $h(\theta; \z)$ as $\widetilde{h}(\theta_{t})$.

\begin{lemma}\label{lem:1} We show that the estimation errors of $\u_t$ and $\v_t$ would be reduced over time. 
\begin{align*}
    \mathbb{E}\left[ \left\| \u_{t+1} - g(\theta_{t+1}) \right\|^2 \right] &\leq (1- \gamma_{t+1}) \left \| \u_t - g(\theta_t) \right \| ^2  + \frac{C^2\eta_t^2}{r_{t+1}}  \left\| \v_t \right \|^2 + \gamma_{t+1}^2 \sigma_g^2; \\
    \mathbb{E}\left[ \left\| \v_{t+1} - \nabla \LL(\theta_{t+1}) \right\|^2 \right] &\leq  (1-\beta_{t+1}) \mathbb{E}\left[\left\| \v_t- \nabla \LL(\theta_{t})\right\|^2 \right] + \frac{2L_0^2+ 18 C^4L^2}{\beta_{t+1}}\eta_t^2\left\| \v_t \right \|^2    \\ & + 18\beta_{t+1}C^2L^2 \left\| \u_{t} - g(\theta_{t}) \right \|^2 + 2 \beta_{t+1}^2 \zeta_h^2 + 6C^2L^2\beta_{t+1}^2\sigma_g^2 + 2 \beta_{t+1}^2 C^2\zeta_g^2.
\end{align*}
\end{lemma}
\begin{proof}
Consider the update $\u_{t+1} = (1-\gamma_{t+1}) \u_t + \gamma_{t+1}\widetilde{g}(\theta_{t+1})$, $\theta_{t+1} = \theta_t - \eta_t \v_t$, and set $p = \frac{\gamma_{t+1}}{1-\gamma_{t+1}}$.
\begin{align*}
&\mathbb{E}\left[ \left\| \u_{t+1} - g(\theta_{t+1}) \right\|^2 \right]  \\
=&\mathbb{E}\left[ \left\|(1-\gamma_{t+1}) (\u_t- g(\theta_{t+1})) + \gamma_{t+1}(\widetilde{g}(\theta_{t+1}) - g(\theta_{t+1})) \right\|^2 \right] \\ 
= & \left\|(1-\gamma_{t+1}) (\u_t- g(\theta_{t+1})) \right \|^2 + \mathbb{E}\left[ \left \| \gamma_{t+1}(\widetilde{g}(\theta_{t+1}) - g(\theta_{t+1})) \right\|^2 \right.  \\
& \left. \quad \quad + \gamma_{t+1}(1-r_{t+1}) (\u_t- g(\theta_{t+1}))(\widetilde{g}(\theta_{t+1}) - g(\theta_{t+1})) \right] \\ 
= & \left\|(1-\gamma_{t+1}) (\u_t- g(\theta_{t+1})) \right \|^2 + \gamma_{t+1}^2 \sigma_g^2 \\ 
\leq & (1-\gamma_{t+1})^2\left\| \u_{t} - g(\theta_{t+1}) \right\|^2 + \gamma_{t+1}^2\sigma_g^2 \\ 
\leq & (1- \gamma_{t+1})^2(1+p) \left \|   \u_{t} - g(\theta_{t}) \right \|^2 + (1-\gamma_{t+1})^2(1+\frac{1}{p})\left \| g(\theta_t) - g(\theta_{t+1}) \right \|^2 +  \gamma_{t+1}^2\sigma_g^2   \\ 
= &  (1- \gamma_{t+1}) \left \| \u_t - g(\theta_t) \right \| ^2  + \frac{C^2\eta_t^2}{r_{t+1}}  \left\| \v_t \right \|^2 + \gamma_{t+1}^2 \sigma_g^2,
\end{align*}
where the last inequality is due to the fact that $(a+b)^2 \leq (1+p)a^2 + (1+\frac{1}{p})b^2$.

According to Assumption~\ref{asm:1}, we know that $\LL(\cdot)$ is also $C_0$-Lipchitz continuous and $L_0$-smooth, where $C_0 = \left(C+1\right)C$ and $L_0 = \left( C^2+C+1\right)L$.
By setting $p = \frac{\beta_{t+1}}{1-\beta_{t+1}}$, we have:

\begin{align*}
&\mathbb{E}\left[ \left\| \v_{t+1} - \nabla \LL(\theta_{t+1}) \right\|^2 \right]\\ 
&=\mathbb{E}\left[ \left\|(1-\beta_{t+1}) (\v_t- \nabla \LL(\theta_{t+1})) + \beta_{t+1}[\nabla \widetilde{h}(\theta_{t+1}) + \nabla f(\u_{t+1})\nabla \widetilde{g}(\theta_{t+1}) -\nabla h(\theta_{t+1})  - \nabla f(g(\theta_{t+1}))\nabla g(\theta_{t+1})] \right\|^2 \right] \\ 
& =\mathbb{E}\left[\left\| (1-\beta_{t+1})  (\v_t- \nabla \LL(\theta_{t})) + (1-\beta_{t+1}) (\nabla \LL(\theta_t) - \nabla \LL(\theta_{t+1}))   + \beta_{t+1}\nabla g(\theta_{t+1})(\nabla f(\u_{t+1}) -\nabla f(g(\theta_{t+1}) ))  \right.\right. \\ & \left. \left. \quad + \beta_{t+1}[\nabla \widetilde{h}(\theta_{t+1}) +\nabla f(\u_{t+1})\nabla \widetilde{g}(\theta_{t+1}) -\nabla h(\theta_{t+1})  - \nabla f(\u_{t+1})\nabla g(\theta_{t+1})] \right\|^2 \right] \\
& \leq \mathbb{E}\left[ \left\|(1-\beta_{t+1})  (\v_t- \nabla \LL(\theta_{t})) + (1-\beta_{t+1}) (\nabla \LL(\theta_t) - \nabla \LL(\theta_{t+1}))  + \beta_{t+1}\nabla g(\theta_{t+1})(\nabla f(\u_{t+1}) -\nabla f(g(\theta_{t+1}) )) \right\|^2  \right] \\ &  \quad + \E \left[2\beta_{t+1}^2 \left\| \nabla \widetilde{h}(\theta_{t+1}) - \nabla h(\theta_{t+1}) \right \|^2+ 2\beta_{t+1}^2 \left \| \nabla f(\u_{t+1})(\nabla \widetilde{g}(\theta_{t+1}) - \nabla g(\theta_{t+1}) )\right \|^2  \right] \\
& \leq  \mathbb{E}\left[ (1-\beta_{t+1})^2 (1+ p ) \left\| \v_t- \nabla \LL(\theta_{t})\right\|^2 +2 (1-\beta_{t+1})^2(1+\frac{1}{p})\left \| \nabla \LL(\theta_t) - \nabla \LL(\theta_{t+1}) \right \|^2 \right] \\ & \quad + 2\beta_{t+1}^2 (1+\frac{1}{p})C^2 \left\|\nabla f(\u_{t+1})- \nabla f(g(\theta_{t+1})) \right \|^2 + 2 \beta_{t+1}^2 \zeta_h^2 + 2 \beta_{t+1}^2 C^2 \zeta_g^2 \\
& \leq  (1-\beta_{t+1}) \mathbb{E}\left[\left\| \v_t- \nabla \LL(\theta_{t})\right\|^2 \right] + \frac{2L_0^2}{\beta_{t+1}}\eta_t^2\left\| \v_t \right \|^2 + 2\beta_{t+1}C^2L^2 \left\| \u_{t+1} - g(\theta_{t+1}) \right \|^2 + 2 \beta_{t+1}^2 \zeta_h^2 + 2 \beta_{t+1}^2 C^2 \zeta_g^2 \\ 
& \ \leq  (1-\beta_{t+1}) \mathbb{E}\left[\left\| \v_t- \nabla \LL(\theta_{t})\right\|^2 \right] + \frac{2L_0^2+ 18 C^4L^2}{\beta_{t+1}}\eta_t^2\left\| \v_t \right \|^2 + 18\beta_{t+1}C^2L^2 \left\| \u_{t} - g(\theta_{t}) \right \|^2 \\ & \quad + 2 \beta_{t+1}^2 \zeta_h^2 + 6C^2L^2\beta_{t+1}^2\sigma_g^2 + 2 \beta_{t+1}^2 C^2\zeta_g^2
\end{align*}
\end{proof}
To finish the proof, we also need to introduce the following lemma.
\begin{lemma} \label{lem:x}\cite{guo2021stochastic} Suppose function $\LL$ is ${L_0}$-smooth and consider the update $\theta_{t+1}:=\theta_{t}-\eta_t \v_{t}$. With $\eta_t L_0 \leq \frac{1}{2}$, we have: 
\begin{equation*}
\begin{split}
 \LL(\theta_{t+1}) \leq \LL(\theta_t) - \frac{\eta_t}{2} \|\nabla \LL(\theta_t)\|^2 + \frac{\eta_t}{2} \|\v_{t} - \nabla \LL(\theta_t)\|^2 - \frac{\eta_t}{4} \Norm{\v_t}^2
\end{split}
\end{equation*}
\end{lemma}
According to above lemmas, we have:
\begin{align*}
    \sum_{t=1}^T \gamma_{t+1} \mathbb{E}\left[ \left\| \u_{t} - g(\theta_{t}) \right\|^2 \right] &\leq  \left \| \u_1 - g(\theta_1) \right \| ^2  + \sum_{t=1}^T\frac{C^2\eta_t^2}{r_{t+1}}  \left\| \v_t \right \|^2 + \sum_{t=1}^T\gamma_{t+1}^2 \sigma_g^2 \\
    \sum_{t=1}^T \beta_{t+1} \mathbb{E}\left[ \left\| \v_{t} - \nabla \LL(\theta_{t}) \right\|^2 \right] &\leq  \mathbb{E}\left[\left\| \v_1- \nabla \LL(\theta_{1})\right\|^2 \right] + \sum_{t=1}^T \frac{2L_0^2+ 18 C^4L^2}{\beta_{t+1}}\eta_t^2\left\| \v_t \right \|^2    \\ & \quad  + 18C^2L^2\sum_{t=1}^T \beta_{t+1} \left\| \u_{t} - g(\theta_{t}) \right \|^2 + \sum_{t=1}^T  \beta_{t+1}^2 \left( 2 \zeta_h^2 + 6C^2L^2\sigma_g^2 + 2 C^2\zeta_g^2 \right) \\
    \sum_{t=1}^{T} \frac{\eta_t}{2} \|\nabla \LL(\theta_t)\|^2  & \leq \LL(\theta_1) + \sum_{t=1}^{T}\frac{\eta_t}{2} \|\v_{t} - \nabla \LL(\theta_t)\|^2 - \sum_{t=1}^{T} \frac{\eta_t}{4} \Norm{\v_t}^2
\end{align*}
Summing up, we have:
\begin{align*}
    &\sum_{t=1}^T \gamma_{t+1}  \left\| \u_{t} - g(\theta_{t}) \right\|^2  +\beta_{t+1}  \left\| \v_{t} - \nabla \LL(\theta_{t}) \right\|^2+ \frac{\eta_t}{2} \|\nabla \LL(\theta_t)\|^2 \\
    \leq & L_1  + \sum_{t=1}^T \left(\frac{C^2 \eta_t} {r_{t+1}} + \frac{(2L_0^2+ 18 C^4L^2) \eta_t}{\beta_{t+1}} -\frac{1}{4}\right) \eta_t \left\| \v_t \right \|^2 + L_1 \sum_{t=1}^T \left(\gamma_{t+1}^2 + \beta_{t+1}^2 \right)\\
    & \quad  + 18C^2L^2\sum_{t=1}^T \beta_{t+1} \left\| \u_{t} - g(\theta_{t}) \right \|^2 + \sum_{t=1}^{T}\frac{\eta_t}{2} \|\v_{t} - \nabla \LL(\theta_t)\|^2, 
\end{align*}
where $L_1 = \max \left\{\left \| \u_1 - g(\theta_1) \right \| ^2 + \left\| \v_1- \nabla \LL(\theta_{1})\right\|^2 + \LL(\theta_1), \sigma^2_g,  2 \zeta_h^2 + 6C^2L^2\sigma_g^2 + 2 C^2\zeta_g^2  \right\}$

By setting $\beta_{t+1} = L_2 \eta_t$ and $\gamma_{t+1} = L_3 \beta_t$, where $L_2 = (16L_0^2+144C^4L^2 +1)$ and $L_3=36C^2(L^2+1)$, we have:
\begin{align*}
    \sum_{t=1}^T \frac{L_2L_3}{2}\eta_{t}  \left\| \u_{t} - g(\theta_{t}) \right\|^2  + \frac{L_2}{2} \eta_{t} \left\| \v_{t} - \nabla \LL(\theta_{t}) \right\|^2+ \frac{\eta_t}{2} \|\nabla \LL(\theta_t)\|^2 
    \leq  L_1 + L_1(L_2^2+L_2^2L_3^2) \sum_{t=1}^T \eta_{t+1}^2 
\end{align*}
According to the analysis of previous works, such as NASA~\cite{Ghadimi2020AST}, since $\sum_{t=1}^{T} \eta_{t+1} = \infty$ a.s., and $\sum_{t=1}^{T} \eta_{t+1}^2 < \infty$, we would have $\u_t \rightarrow g(\theta_t)$, $\v_t \rightarrow \nabla \LL(\theta_t)$, $\|\nabla \LL(\theta_t)\|^2 \rightarrow  0$, otherwise the left part of above equation tends to infinity, while the right part is not.

\subsection{Proof of Theorem~\ref{thm:2}}
With supporting lemmas in previous theorem, by setting $\beta_t = \gamma_t =\beta, \eta_t = \eta$, we have:
\begin{align*}
\mathbb{E}\left[\beta \sum_{t=1}^T\left\| \v_{t} - \nabla \LL(\theta_t) \right\|^2 \right] \leq & \mathbb{E}\left[  \left \| \v_1 - \nabla \LL(\theta_1) \right \|^2 \right] + \frac{2L_0^2 + 18C^4L^2}{\beta} \sum_{t=1}^T \eta^2 \left\| \v_t \right \|^2 \\ & + 18\beta C^2L^2\sum_{t=1}^T \| \u_t - g(\theta_t)\|^2  + (2 \zeta_h^2+ 6C^2L^2\sigma_g^2 + 2  C^2 \zeta_g^2) \beta^2T,
\end{align*}
where $\mathbb{E}\left[  \left \| \v_1 - \nabla \LL(\theta_1) \right \|^2 \right] \leq 3\zeta_h^2+3C^2\zeta_g^2 + 3C^2L^2\sigma_g^2$.

Similarly, for the term $\| \u_t - g(\theta_t)\|^2$, we have:
\begin{align*}
\sum_{t=1}^T \| \u_t - g(\theta_t)\|^2 & \leq \frac{1}{\beta} \left \| \u_1 - g(\theta_1) \right \|^2 + \frac{C^2\eta^2}{\beta^2} \sum_{t=1}^T \left\| \v_t \right \|^2+\beta \sigma_g^2 T \leq \frac{\sigma_g}{\beta} + \frac{C^2\eta^2}{\beta^2}\sum_{t=1}^T \left\| \v_t \right \|^2 + \beta \sigma_g^2 T.
\end{align*}

So, we have
\begin{align*}
&\mathbb{E}\left[\sum_{t=1}^T\left\| \v_{t} - \nabla \LL(\theta_t) \right\|^2 \right] \\
& \leq \frac{3\zeta_h^2+3C^2\zeta_g^2 +21C^2L^2\sigma_g^2}{\beta} + \frac{2L_0^2+ 36C^4L^2}{\beta^2} \eta^2 \sum_{t=1}^T \left\| \v_t \right \|^2  + (2\zeta_h^2  + 2C^2\zeta_g^2 + 24C^2L^2\sigma_g^2)\beta T
\\& = \frac{L_2}{\beta} + \frac{L_4}{\beta^2}\eta^2 \sum_{t=1}^T \left\| \v_t \right \|^2 + L_3\beta T,
\end{align*}
where $L_2$, $L_3$, $L_4$ is the constant.

Define the $\beta = \min\{\frac{\epsilon^2}{3L_3},1\},\eta = \min \{ \frac{\epsilon^2}{3\sqrt{2L_4}L_3}, \frac{1}{2L_0}, \frac{\beta}{\sqrt{2 L_4}} \}, T =\max \{\frac{L_2}{\beta \epsilon^2}, \frac{2\LL(\theta_1)}{\eta \epsilon^2} \}$. We can have:

\begin{align*}
\mathbb{E}\left[\frac{1}{T} \sum_{t=1}^T \| \nabla \LL(\theta_t) \|^2\right] & \leq \frac{2\LL(\theta_1)}{\eta T} + \frac{1}{T}\mathbb{E}\left[\sum_{t=1}^T\left\| \v_{t} - \nabla \LL(\theta_t) \right\|^2 \right] - \frac{1}{2T} \sum_{t=1}^T \left\| \v_t \right \|^2 \\
& \leq \frac{2\LL(\theta_1)}{\eta T} + \frac{L_2}{\beta T} + \left(\frac{L_4\eta^2}{\beta^2T} - \frac{1}{2T}\right) \sum_{t=1}^T \left\| \v_t \right \|^2 + L_3\beta \\ & \leq \epsilon^2.
\end{align*}
To hide the constant, we finally have $T=\max\left\{\mathcal{O}\left(\frac{1}{\epsilon^{2}}\right), \mathcal{O}\left(\frac{\zeta_h^2+\zeta_g^2 +\sigma_g^2}{\epsilon^{4}}\right)\right\}$. We can easily extend the results to the case of adaptive learning rates in the style of Adam, following the analysis of ~\citet{guo2021stochastic}.


\subsection{Proof of Theorem~\ref{thm:3}}
If $\LL(\theta)$ is $\mu$-PL, which is $2\mu(\LL(\theta) - \LL^*) \leq \| \nabla \LL(\theta)\|^2$, we can further improve the rate.

Define that $\gamma_{t+1} = \beta_{t+1}, \Gamma_t = \LL(\theta_t) - \LL^{*} + A \|\v_t - \nabla \LL(\theta_t) \|^2 + B \|\u_t - g(\theta_t) \|^2$, and then we have:

\begin{align*}
 &\mathbb{E}\left[ \Gamma_{t+1} \right] \\
 \leq & (1-\mu\eta_t) \mathbb{E}\left[ \LL(\theta_{t}) - \LL^{*} \right]+ \frac{\eta_t}{2} \| \v_t-\nabla \LL(\theta_t) \|^2 - \frac{\eta_t}{4} \| \v_t \|^2 + A(1-\beta_{t+1})\|\v_t - \nabla \LL(\theta_t) \|^2 \\ 
 & \quad + \frac{2L_0^2+18C^4L^2}{\beta_{t+1}}A \eta_t^2 \|\v_t \|^2 + 18\beta_{t+1}C^2L^2A \|\u_t - g(\theta_t) \|^2  + A(2\zeta_h^2 + 6C^2L^2\sigma_g^2 + 2C^2\zeta_g^2)\beta_{t+1}^2 \\ 
 & \quad + B(1-\beta_{t+1}) \| \u_t - g(\theta_t)\|^2 + \frac{BC^2\eta_t^2}{\beta_{t+1}}\| \v_{t} \|^2 + B \beta_{t+1}^2\sigma_g^2
\end{align*}

Setting $B = 36C^2L^2A, \beta_{t+1} = \max \{2\mu , \frac{1}{A}, 8A(2L_0^2 + 36C^4L^2) \}\eta_t$, we have:

\begin{align*}
 \mathbb{E}\left[ \Gamma_{t+1} \right] & \leq (1-\mu\eta_t) \mathbb{E}\left[ \Gamma_t \right] + A\beta_{t+1}^2 (2\zeta_h^2 + 42C^2L^2\sigma_g^2 + 2C^2\zeta_g^2)
\end{align*}

Define $L' = 2\zeta_h^2 + 42C^2L^2\sigma_g^2 + 2C^2\zeta_g^2, L'' = \max\{8(2L_0^2 + 54C^4L^2),2\}$, we have:
\begin{align*}
    \mathbb{E}\left[ \Gamma_{t+1} \right]  \leq (1-\mu\eta_t) \mathbb{E}\left[ \Gamma_t \right] + A\beta_{t+1}^2L'.
\end{align*}

If $\mu \leq 1$, set $A=1, \beta_{t+1} = L'' \eta_t , 1 -\mu \eta_t = \frac{\eta_t^2}{\eta_{t-1}^2}$, which is $\frac{1}{\eta_t} = \frac{\mu}{2} + \sqrt{\frac{\mu^2}{4} + \frac{1}{\eta_{t-1}^2}}$.

\begin{align*}
 \mathbb{E}\left[ \Gamma_{T+1} \right] & \leq (1-\mu\eta_T) \mathbb{E}\left[ \Gamma_T \right] + L' (L'')^2\eta_T^2
 \\ & \leq \frac{\eta_T}{\eta_{T-1}}\mathbb{E}\left[ \Gamma_{T} \right] + L' (L'')^2\eta_T^2
 \\ &\leq \frac{\eta_T}{\eta_{T-2}}\mathbb{E}\left[ \Gamma_{T} \right] + L' (L'')^2\eta_T^2*2
  \\ &\leq \frac{\eta_T}{\eta_{0}}\mathbb{E}\left[ \Gamma_{T} \right] + L' (L'')^2\eta_T^2 *T
\end{align*}
We have $\frac{1}{\eta_t} \geq \frac{\mu}{2} + \frac{1}{\eta_{t-1}}$, so we have $\frac{1}{\eta_T} \geq \frac{\mu T}{2} + \frac{1}{\eta_{0}}$, where $\eta_0 > 0$ Finally, we have $\eta_T \leq \frac{2}{\mu T}$. 

So $ \mathbb{E}\left[ \Gamma_{T+1} \right] \leq  \frac{4}{\eta_0^2 \mu^2 T^2}\mathbb{E}\left[ \Gamma_{1} \right] + L'(L'')^2\frac{4}{\mu^2T} \leq \epsilon$. The complexity is $T = \max \left\{ \O( \frac{1}{\mu \eta_0 \sqrt{\epsilon}}),\O(\frac{\zeta_h^2+\zeta_g^2+\sigma_g^2}{\mu^2 \epsilon}) \right\}$.

If $\mu \geq 1$, set $A = \frac{1}{\mu},\beta_{t+1} = L''\mu \eta_t$, and the complexity is $T = \max \left\{ \O( \frac{1}{\mu \eta_0 \sqrt{\epsilon}}),\O(\frac{\zeta_h^2+\zeta_g^2+\sigma_g^2}{\mu \epsilon}) \right\}$.

Note that existing methods usually employ two-loop stage-wise designed algorithms to analyze for PL condition~\cite{pmlr-v162-wang22ak, jiang2022multiblocksingleprobe}.

\subsection{Proof of Lemma~\ref{lem:3}}

We analyze the variance $\sigma_g^2$ of our estimator $g(\theta)$. By setting $\mu = g(\theta) = \int  p_0(\widetilde \z;\theta) d \widetilde\z = \E_{\widetilde\z\sim q(\widetilde\z)}\left[\frac{p_0(\widetilde \z;\theta)}{q(\widetilde\z)}\right]$, the variance of our estimator is 
\begin{align*}
\E_{\widetilde\z\sim q(\widetilde\z)} \Bigg\| \frac{p_0(\widetilde \z;\theta)}{q(\widetilde\z)} - \mu \Bigg\|^2 
=& \int \left(\frac{p_0(\widetilde \z;\theta)}{q(\widetilde\z)}-\mu\right)^{2} q(\widetilde\z) d\widetilde\z
= \int \frac{( p_0(\widetilde \z;\theta) -\mu q(\widetilde\z))^{2}}{q(\widetilde\z)} d \widetilde\z
\end{align*}

To get the minimum variance, we should choose $q(\widetilde\z) = \frac{ p_0(\widetilde \z;\theta)}{\mu} = \frac{ p_0(\widetilde \z;\theta)}{\int  p_0(\widetilde \z;\theta) d \widetilde\z} = p_\theta (\widetilde \z)$ and the variance would be zero. 

\subsection{Proof of Proposition~\ref{prop:2}}
First, we write the MLE objective function as:
\begin{align*}
\LL(\theta) &= \E_{\x\sim p_\text{data}}\left[-\theta x + \frac{1}{2}x^2\right] + \log \int e^{\theta x - \frac{1}{2}x^2}d\x\\
& = \E_{\x\sim p_\text{data}}\left[-\theta x + \frac{1}{2}x^2\right] + \log \sqrt{2\pi} +\frac{\theta^2}{2}.
\end{align*}
So, we have:
\begin{align*}
L(\theta) -L(\theta^{*}) &= \E_{\x\sim p_\text{data}}[(\theta^{*}-\theta) \x ] +\frac{\theta^2 - {\theta^{*}}^2}{2}= (\theta^{*}-\theta) \theta^{*} +\frac{\theta^2 - {(\theta^{*})}^2}{2}=  (\theta-\theta^{*})^2.
\end{align*}

\subsection{Proof of Proposition~\ref{prop:3}}
We prove that the objective is $1$-strongly convex, which is a stronger condition than $1$-PL condition. 

Assume that $\mu = \int x p_{\theta}(x) d\x $ and we have:
\begin{align*}
&{\nabla}^2 L(\theta) \\
&=\frac{\left(\int x^2 e^{\theta x - \frac{1}{2}x^2}d\x \right) \left(\int  e^{\theta x - \frac{1}{2}x^2}d\x\right) - \left(\int x e^{\theta x - \frac{1}{2}x^2}d\x\right)^2}{\left[\int  e^{\theta x - \frac{1}{2}x^2}d\x\right]^2} \\
&=\frac{\int x^2 e^{\theta x - \frac{1}{2}x^2}d\x  }{\int  e^{\theta x - \frac{1}{2}x^2}d\x} -  \left[\frac{\int x e^{\theta x - \frac{1}{2}x^2}d\x}{\int  e^{\theta x - \frac{1}{2}x^2}d\x}\right]^2 \\
&=\int x^2 p_{\theta}(x)d\x   -  \left[\int x p_{\theta}(x) d\x \right]^2 \\
&=\int x^2 p_{\theta}(x)d\x   -  2 \mu \int x p_{\theta}(x) d\x  + \mu^2 \\
&=\int x^2 p_{\theta}(x)  -  2 \mu  x p_{\theta}(x)   + \mu^2 p_{\theta}(x) d\x \\
&=\int (x-\mu)^2 p_{\theta}(x)  d\x \\
&= \E_{p(\theta)} \left[ (x-\mu)^2 \right] \\
&= 1
\end{align*}
The last equation is due to the fact that the variance is fixed as 1 for the model.
So, the objective is $1$-strongly convex, and according to our analysis for PL condition, our method enjoys a complexity of $\mathcal{O}(\epsilon^{-1})$ due to Theorem~\ref{thm:3}.


\end{document}